\documentclass[10pt,twocolumn,letterpaper]{article}

\usepackage{cvpr}
\usepackage{times}
\usepackage{epsfig}
\usepackage{amsmath}
\usepackage{amssymb}
\usepackage{amsthm}
\usepackage{upgreek}
\usepackage{here}
\usepackage{booktabs}
\usepackage{subfig}
\usepackage{multirow}

\graphicspath{{./fig/}}

\usepackage{algorithm,algpseudocode}
\usepackage{algorithmicx}
\algnewcommand\algorithmicinput{\textbf{Input:}}
\algnewcommand\Input{\item[\algorithmicinput]}
\algnewcommand\algorithmicoutput{\textbf{Output:}}
\algnewcommand\Output{\item[\algorithmicoutput]}
\algnewcommand\algorithmicinitialize{\textbf{Initialize:}}
\algnewcommand\Initialize{\item[\algorithmicinitialize]}

\newtheorem{definition}{Definition}[section]

\newcommand{\argmin}{\mathop{\rm arg~min}\limits}
\newcommand{\bvec}[1]{\boldsymbol{\mathbf{#1}}}

\newcommand{\bmat}[1]{\boldsymbol{\mathbf{#1}}}

\newcommand{\Tref}[1]{Table~\ref{#1}}
\newcommand{\Eref}[1]{Eq.~(\ref{#1})}
\newcommand{\Fref}[1]{Fig.~\ref{#1}}
\newcommand{\Aref}[1]{Alg.~\ref{#1}}
\newcommand{\Sref}[1]{Section~\ref{#1}}

\newtheorem{theorem}{Theorem}

\usepackage[breaklinks=true,bookmarks=false]{hyperref}
\cvprfinalcopy
\def\cvprPaperID{2506} % *** Enter the CVPR Paper ID here

\ifcvprfinal\fi
\begin{document}

\setlength{\abovedisplayskip}{2pt} % ???????
\setlength{\belowdisplayskip}{2pt} % ???????

%%%%%%%%% TITLE
\title{Residual Expansion Algorithm: Fast and Effective Optimization for \\ Nonconvex Least Squares Problems}
	
\author{Daiki Ikami~~~~~~Toshihiko Yamasaki~~~~~~Kiyoharu Aizawa \\
	The University of Tokyo, Japan\\
	{\tt\small\{ikami, yamasaki, aizawa\}@hal.t.u-tokyo.ac.jp}
}

\maketitle
%\thispagestyle{empty}

	%%%%%%%%% ABSTRACT
	\begin{abstract}
We propose the residual expansion (RE) algorithm: a global (or near-global) optimization method for nonconvex least squares problems. Unlike most existing nonconvex optimization techniques, the RE algorithm is not based on either stochastic or multi-point searches; therefore, it can achieve fast global optimization. Moreover, the RE algorithm is easy to implement and successful in high-dimensional optimization. The RE algorithm exhibits excellent empirical performance in terms of k-means clustering, point-set registration, optimized product quantization, and blind image deblurring.
	\end{abstract}
	
	%%%%%%%%% BODY TEXT
	\section{Introduction}
	Many problems in computer vision and machine learning can be formulated as optimization problems. If we can formulate a problem as a convex optimization, we can solve it by convex optimization techniques such as gradient-based methods. However, most optimization problems are nonconvex and often have many local minima. In these cases, convex optimization techniques can find only local minima.
	
	Global optimization of nonconvex problems is an NP-hard problem in most cases. Therefore, heuristic methods are often used to find a global (or near-global) optimum. There are two major approaches: good initialization and stochastic optimization. The former is fast and effective if we can obtain a good initial guess~\cite{arthur2007k}; however, many optimization problems do not provide this. The latter is random search or multiple-point search, which is represented by simulated annealing (SA)~\cite{kirkpatrick1983optimization}, particle swarm optimization (PSO)~\cite{kennedy2011particle}, and genetic algorithms (GA)~\cite{mitchell1998introduction}. Although these methods are effective with low-dimensional optimization problems, it is difficult to obtain good solutions with high-dimensional ones. Moreover, these approaches often require excessive computation time to obtain a good solution.
	
	\begin{figure}[t]
		\begin{center}
			\hfill
			\subfloat[The global optimum result.]{\includegraphics[width=0.45\linewidth]{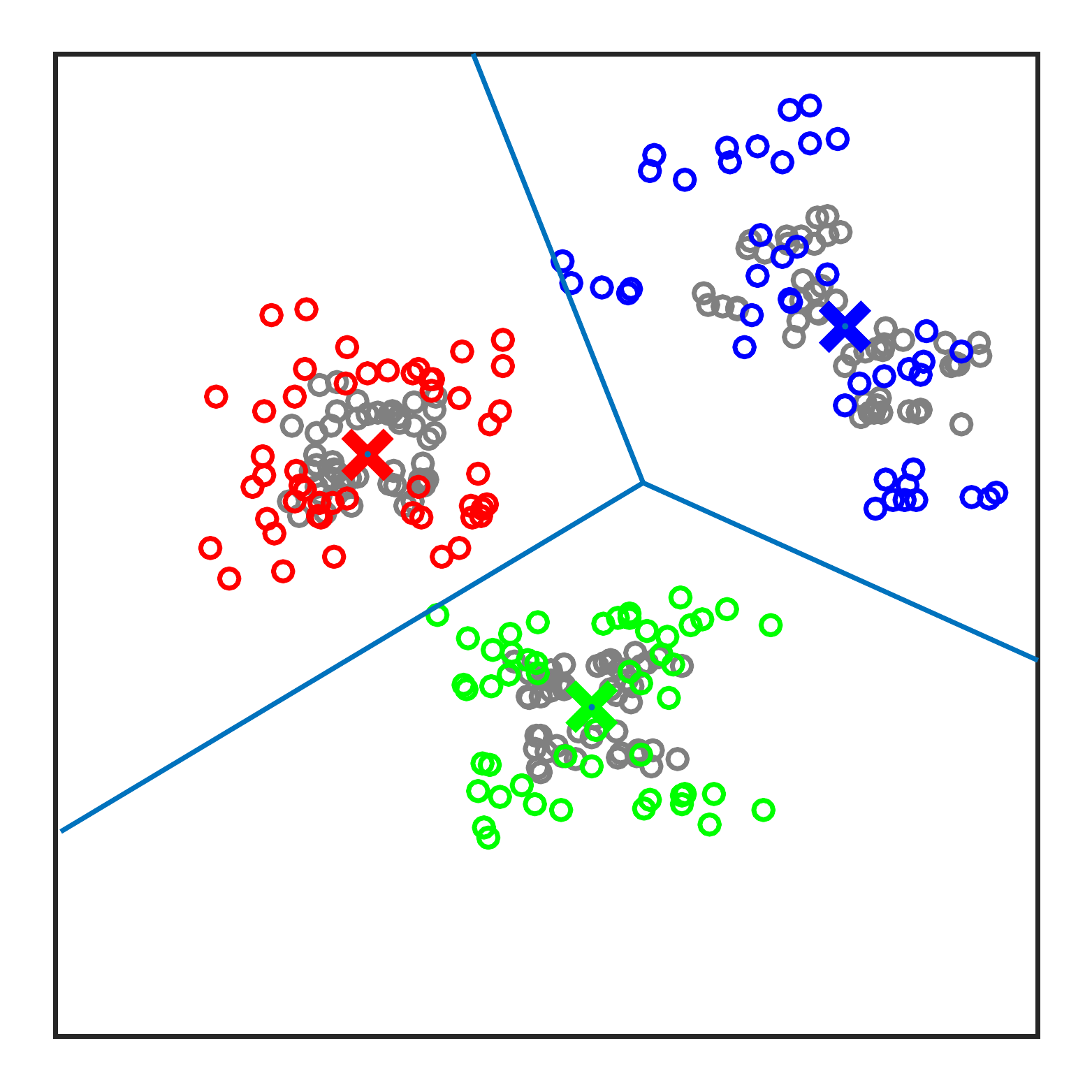}
				\label{fig:RE_fig1_1}} 
			\hfill
			\subfloat[The local minimum result.]{\includegraphics[width=0.45\linewidth]{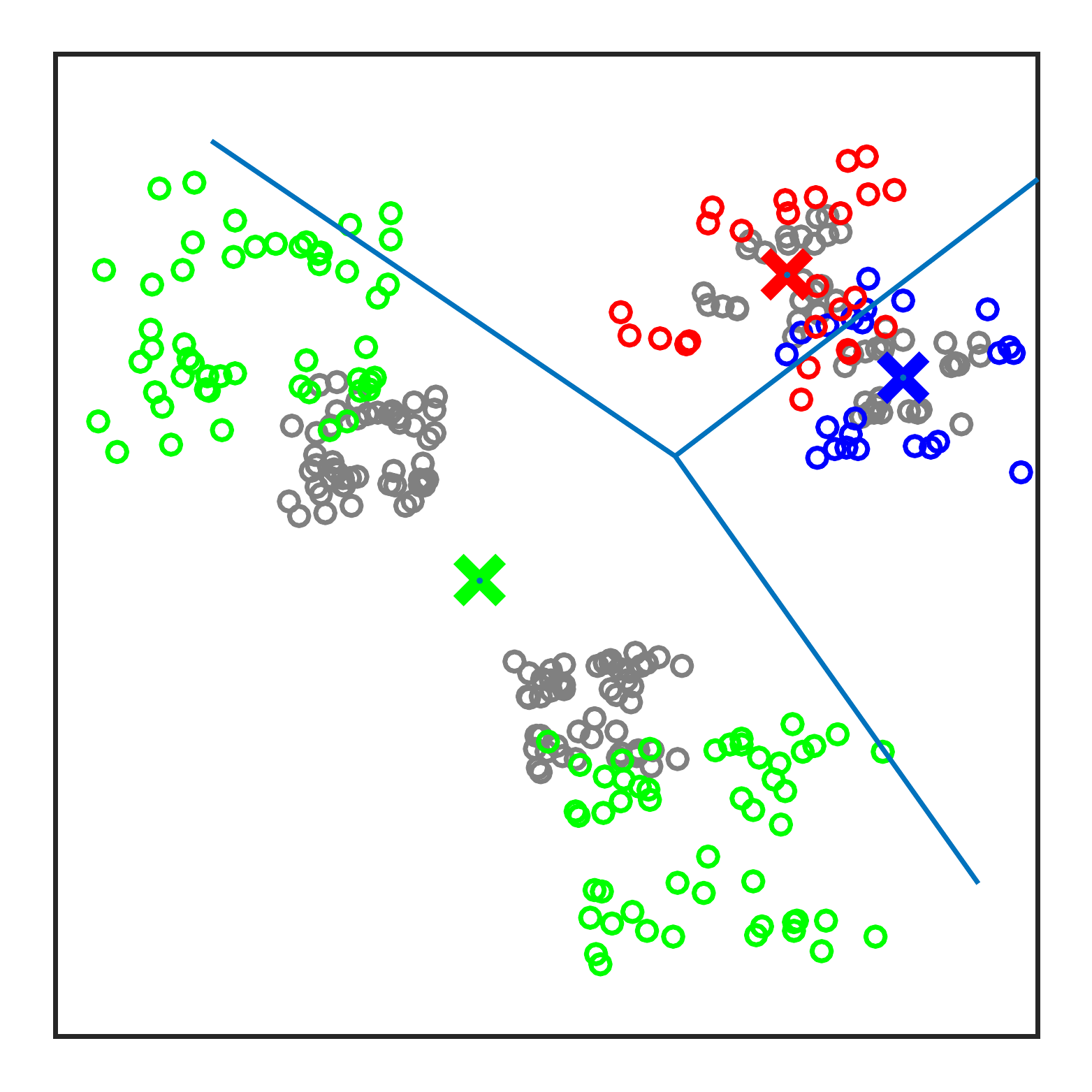}
				\label{fig:RE_fig1_2}} 
			\hfill\null
		\end{center}
		\vspace{-4mm}
		\caption{K-means clustering results with different RE convergence. Gray circles denote original data points and red, blue, and green circles denote $\alpha$-expanded data points from each cluster center. \Fref{fig:RE_fig1_1} shows $\alpha$ RE convergence with $\alpha=1$; however, \Fref{fig:RE_fig1_2} does not show this. In this case, the solution with a larger RE constant achieves the global optimum. The details of RE convergence and the RE constant are described in \Sref{sec:REConvergence}.}
		\label{fig:RE_fig1}
		\vspace{-4mm}
	\end{figure}
	
	In this paper, we propose a fast and effective optimization method for nonconvex least squares (LS) problems such as k-means clustering and point-set registration. First, we propose a novel measure of convergence called RE convergence: this represents how far we can expand data points along their residual directions under convergence.
	\Fref{fig:RE_fig1} shows k-means results and expanded data points. \Fref{fig:RE_fig1_1} depicts convergence on expanded data while \Fref{fig:RE_fig1_2} shows a case that is not converged. We presume that RE convergence is associated with global convergence. In fact, we can prove that the solution that is stable on a large expansion is the global optimum in the case of a one-dimensional quartic minimization problem.
	
	Additionally, we propose a heuristic algorithm to find a solution that is stable on the large expansion, which we term the residual expansion (RE) algorithm. This algorithm is based on neither multiple-point search nor random search, and thus fast computation can be achieved. 

	Our contribution is as follows: 
	\begin{enumerate}
		\setlength{\parskip}{0cm}
		\setlength{\itemsep}{0cm}
		\item{We propose a novel concept of convergence, RE convergence. We show the relationship between RE convergence and the global optimum.}
		\item{We propose the RE algorithm, which can be applied for any nonconvex LS problem. We show that the RE algorithm is fast, effective, and easy to implement.}
		\item{We show the RE algorithm's excellent performance for various nonconvex LS problems such as k-means clustering, point-set registration, optimized product quantization, and blind image deblurring. }
	\end{enumerate}

	\section{Related works}
	\subsection{Nonconvex least squares problems}
	\label{sec:nonconvexLS}
	We focus on nonconvex LS problems, of which many exist. In this paper, we study the following four important problems in computer vision and machine learning.
	
	\subsubsection{K-means clustering}
	K-means clustering is one of the most popular clustering methods with various applications such as quantization~\cite{jegou2011product}, feature learning~\cite{csurka2004visual}, and segmentation~\cite{achanta2012slic}. K-means clustering assigns data vectors $\bvec{x}_1,\ldots,\bvec{x}_n\in\mathbb{R}^d$ to the nearest representative clusters. The optimization problem can be formulated as
	\begin{gather}
		\min_{\bmat{C},\bmat{Z}} \frac{1}{2}\lVert\bmat{X}-\bmat{C}\bmat{Z}\rVert_F^2 
		\label{eq:kmeans} \\ \mbox{s.t.}\;\; z_{ij} = \{0,1\}, \lVert\bvec{z}_i\rVert_1 = 1,\nonumber
	\end{gather}
	where $\bmat{X}=[\bvec{x}_1,\ldots,\bvec{x}_n]\in\mathbb{R}^{d\times n}$ is a data matrix, $\bmat{C}=[\bvec{c}_1,\ldots,\bvec{c}_k] \in \mathbb{R}^{d\times k}$ is a matrix of cluster centroids, and $\bmat{Z} = [\bvec{z}_1,\ldots,\bvec{z}_n]\in\mathbb{R}^{k\times n}$ is an assignment matrix.
	
	The most popular optimization method is Lloyd's algorithm~\cite{lloyd1982least}, which has an update step (fix $\bmat{Z}$ and update $\bmat{C}$) and an assignment step (fix $\bmat{C}$ and update $\bmat{Z}$). Hartigan's algorithm~\cite{hartigan1975clustering} achieves better clustering than Lloyd's algorithm because the set of local minima of Hartigan's algorithm is a subset of those of Lloyd's algorithm~\cite{telgarsky2010hartigan,slonim2013hartigan}. For good initialization, k-means++~\cite{arthur2007k} is often used because of its efficiency and effectiveness.

	\subsubsection{Point-set registration}
	Point-set registration is a fundamental problem in computer vision. Here we consider a rigid 3D-point-set registration problem: Given source point sets $\bmat{X}=[\bvec{x}_1,\dots,\bvec{x}_n]\in\mathbb{R}^{3\times n}$ and target point sets $\bmat{Y}=[\bvec{y}_1,\ldots,\bvec{y}_m] \in \mathbb{R}^{3\times m}$, we estimate the best rigid transformation parameters.
	
	In this paper, we consider the following optimization problem with a point-to-point cost function:
	\begin{gather}
		\min_{\bmat{R},\bvec{t},\bmat{Z}} \frac{1}{2}\lVert\bmat{R}\bmat{X}+\bvec{t}\bvec{1}^\top-\bmat{Y}\bmat{Z}\rVert_F^2\label{eq:ICP} \\ \mbox{s.t.} \;\; z_{ij} = \{0,1\}, \lVert\bvec{z}_i\rVert_1 = 1, \bmat{R}^\top\bmat{R}=\bmat{I}, \nonumber
	\end{gather}
	where $\bmat{R}\in\mbox{SO}(3)$ is a rotation matrix, $\bvec{t}\in\mathbb{R}^3$ is a translation vector, and $\bmat{Z} = [\bvec{z}_1,\ldots,\bvec{z}_n]\in\mathbb{R}^{k\times n}$ is an assignment matrix. $\bmat{I}$ is an identity matrix and $\bvec{1}$ is a vector of all ones.
	
	The iterative closest point (ICP) algorithm~\cite{besl1992method} is a well-known alternating optimization method: it fixes $\bmat{Z}$ and updates $\bmat{R}$, $\bvec{t}$, and then fixes $\bmat{R}, \bvec{t}$ and updates $\bmat{Z}$. To obtain a global minimum, some studies adopt stochastic optimization, such as GA~\cite{silva2005precision}, PSO~\cite{wachowiak2004approach}, and SA~\cite{luck2000registration}. Recently, Yang \etal proposed Go-ICP~\cite{yang2013go}, which guarantees global optimality by using the branch-and-bound algorithm. However, it requires significant computation time.

	\subsubsection{Optimized product quantization}
	Optimized product quantization (OPQ)~\cite{ge2013optimized,norouzi2013cartesian}, which is an extension of product quantization (PQ), is an efficient fast approximate nearest neighbor search method. The optimization problem in OPQ is described by
	\begin{eqnarray}
		\min_{\bmat{R},\bmat{C},\bmat{Z}}\frac{1}{2}\sum_{i=1}^{N}\left\lVert\bvec{x}_i-\bmat{R}\left[
		\begin{array}{c}\bmat{C}^{(1)}\bvec{z}_i^{(1)} \\
			\vdots \\
			\bmat{C}^{(M)}\bvec{z}_i^{(M)}\end{array}
		\right]\right\rVert_2^2\label{eq:OPQ} \label{eq:OPQ} \\ \mbox{s.t.} \;\; z_{ij}^{(m)} = \{0,1\}, \left\lVert\bvec{z}_i^{(m)}\right\rVert_1 = 1,\bmat{R}^\top\bmat{R}=\bmat{I}, \nonumber
	\end{eqnarray}
	where $\bmat{X},\bmat{C},\bmat{Z}$ have the same meaning as in Section 2.1.1 and $\bmat{R}$ is a rotation matrix.
	
	The optimization problem of \Eref{eq:OPQ} can be solved by alternating optimization of $\bmat{R}$, $\bmat{C}$, and $\bmat{Z}$~\cite{ge2013optimized,norouzi2013cartesian}. Ge \etal also proposed a parametric optimization method that assumes the data follows a parametric Gaussian distribution~\cite{ge2013optimized}.

	\subsubsection{Blind image deblurring}
	Blind image deblurring has long been a challenging problem in computer vision. 
	From a blurred image $\bmat{B}\in\mathbb{R}^{h\times w}$, we estimate an original image $\bmat{I}\in\mathbb{R}^{h\times w}$ and blur kernel $\bmat{k}\in\mathbb{R}^{k\times k}$ by minimizing the following cost function: 
	\begin{equation}
		\min_{\bmat{I},\bmat{k}}\frac{1}{2}\lVert\bmat{I}\otimes\bmat{k}-\bmat{B}\rVert_F^2 + \gamma_I R_I(\bmat{I}) + \gamma_k R_k(\bmat{k}),
		\label{eq:BlindImageDeconv}
	\end{equation}
	where $R_I(\bmat{I})$ and $R_k(\bmat{k})$ are the regularization terms, and $\otimes$ denotes the convolution operator. For $R_I(\bmat{I})$, L0-norm (or approximately L0-norm)~\cite{xu2013unnatural,pan2014deblurring}, or L1/L2 functions~\cite{krishnan2011blind} are proposed to enforce the sharp edges of the original image. For $R_k(\bmat{k})$, L2-norm~\cite{xu2013unnatural,pan2014deblurring} or L1-norm~\cite{krishnan2011blind} are often used. We refer to the paper~\cite{Lai_2016_CVPR} for a recent comparative study of blind image deblurring.
	
	We can minimize \Eref{eq:BlindImageDeconv} by alternating optimization of $\bmat{I}$ and $\bmat{k}$. For fast and effective optimization, a coarse-to-fine strategy~\cite{cho2009fast,krishnan2011blind,xu2013unnatural,pan2014deblurring} is generally employed.

	\subsection{Nonconvex optimization techniques}
	\label{sec:RelatedNonconvex}
	Most nonconvex optimization techniques are based on stochastic optimization, including GA~\cite{mitchell1998introduction}, PSO~\cite{kennedy2011particle}, and SA~\cite{kirkpatrick1983optimization}. These methods generally do not work well or require significant computation time for high-dimensional optimization problems. Several studies~\cite{blais1995registering,krishna1999genetic,luck2000registration,wachowiak2004approach,silva2005precision} have employed these nonconvex optimization techniques to our target problems described in \Sref{sec:nonconvexLS}; however, these methods are not often used in practice.
	
	Our approach is related to graduated nonconvexity (GNC)~\cite{blake1987visual}, which first solves a simplified optimization problem and then gradually transforms the problem into the original nonconvex problem. The basic concept of graduated optimization methods is {\it extinguishing local minima} by using a convexified objective function, and then gradually changing the objective function to the original function. We refer readers to ~\cite{mobahi2015link} for a recent survey of graduated optimization. In contrast to GNC, our approach is explicitly {\it to escape from poor local minima} by using a largely expanded objective function and then gradually transforming it into the original function, as described in \Sref{sec:REalgorithm}.

	\section{Residual expansion convergence}
	\label{sec:REConvergence}
	First, we describe RE convergence, which indicates how we can expand data along their residual directions. RE convergence is a measure of the depth of convergence, and our proposed algorithm is based on this concept. We show a relationship between the global optimum and RE convergence.
	
	We discuss a nonconvex least squares (LS) optimization problem whose objective function is given by
	\begin{equation}
		E\left(\bvec{\uptheta}\right) = \frac{1}{2}\lVert\bvec{y}-\bvec{f}\left(\bvec{\uptheta}\right)\rVert_2^2.
		\label{eq:LS}
	\end{equation}
	Our definitions are as follows.
	\begin{definition}[Residual Expansion]
		Let $\bvec{\uptheta}^*$ be a local minimum point of \Eref{eq:LS}. We define the $\alpha$-expanded objective function $E_\alpha(\bvec{\uptheta})$:
		\begin{eqnarray}
			E_\alpha\left(\bvec{\uptheta}\right) = \frac{1}{2}\lVert\hat{\bvec{y}}-\bvec{f}\left(\bvec{\bvec{\uptheta}}\right)\rVert_2^2.	
		\end{eqnarray}
		where $\hat{\bvec{y}}$ is constructed by expanding $\bvec{y}$ in the residual direction with a magnitude of $\alpha$ as
		\begin{equation}
			\hat{\bvec{y}} = \bvec{y}+\alpha\left(\bvec{y}-\bvec{f}\left(\bvec{\bvec{\uptheta}^*}\right)\right), \\
		\end{equation}
		We call the operation that constructs the $\alpha$-expanded objective function residual expansion (with $\alpha$).
	\end{definition}
	\begin{definition}[$\alpha$ RE convergence]
		$\bvec{\uptheta}^*$ is called $\alpha$ RE convergence if there exists a constant $\alpha\geq0$ such that $\bvec{\uptheta}$ is still a local optimum of $E_\alpha(\bvec{\uptheta})$. In particular, the maximum (or supremum) constant is called the RE constant\footnote{If $\bvec{\uptheta}^*$ is always a local minimum of $E_\alpha(\bvec{\uptheta}^*)$ under all $\alpha\geq0$, we define the RE constant as $\infty$. }.	
	\end{definition}
	
	Our hypothesis is that {\it a solution with a larger $\alpha$-RE constant is closer to the global optimum solution}. This hypothesis holds in the case of quartic minimization, as presented in section 3.1.1.

	\subsection{Unconstrained and differentiable problems}
	We consider one of the simplest cases: unconstrained and differentiable LS problems. Given a local optimum $\bvec{\uptheta}^*$, we can obtain first- and second-order derivatives of the $\alpha$-expanded objective function $E_\alpha\left(\bvec{\uptheta}\right)$ at $\bvec{\uptheta^*}$ as
	\begin{eqnarray}
		\nabla E_\alpha\left(\bvec{\uptheta}^*\right) = (1+\alpha)\bmat{J}^\top\left(\bvec{\uptheta}^*\right)\left(\bvec{y}-\bvec{f}\left(\bvec{\uptheta}^*\right)\right) = \bvec{0},
		\label{eq:RE_derivative1}\\
		\nabla^2 E_\alpha\left(\bvec{\uptheta}^*\right) = \bvec{J}^\top\left(\bvec{\uptheta}^*\right)\bvec{J}\left(\bvec{\uptheta}^*\right) + (1+\alpha)\bmat{S}\left(\bvec{\uptheta}^*\right).
		\label{eq:RE_derivative2}
	\end{eqnarray}
	where $J$ is a Jacobian matrix and $\bmat{S}\left(\bvec{\uptheta}^*\right)$ is
	\begin{equation}
		\bmat{S}\left(\bvec{\uptheta}^*\right) = \sum_i{\nabla^2f_i(\bvec{\uptheta}^*)\left(y_i-f_i\left(\bvec{\uptheta}^*\right)\right)}.
	\end{equation}
	\Eref{eq:RE_derivative1} means that $\bvec{\uptheta}^*$ is always a stationary point of $E_\alpha\left(\bvec{\uptheta}\right)$. Therefore, $\bvec{\uptheta}^*$ is a local minimum of $E_\alpha\left(\bvec{\uptheta}\right)$ if and only if $\nabla^2 E_\alpha\left(\bvec{\uptheta}\right)$ is a positive semi-definite (PSD) matrix. If $\bmat{S}$ is not a PSD matrix, there is a $\alpha\geq 0$ which satisfies the fact that $\nabla^2 E_\alpha\left(\bvec{\uptheta}\right)$ is not a PSD matrix. 
	
	\Fref{fig:REConvergence} shows examples of $\alpha$-expanded objective functions. Residual expansion elevates the objective function around $\uptheta^*$, and if $\alpha$ is sufficiently large then it ceases to be a local minimum.
	
	\begin{figure}[t]
		%	\begin{center}
		\centering
		\subfloat[$\theta_1^*$ and its expanded objective function.]{\includegraphics[width=0.45\linewidth]{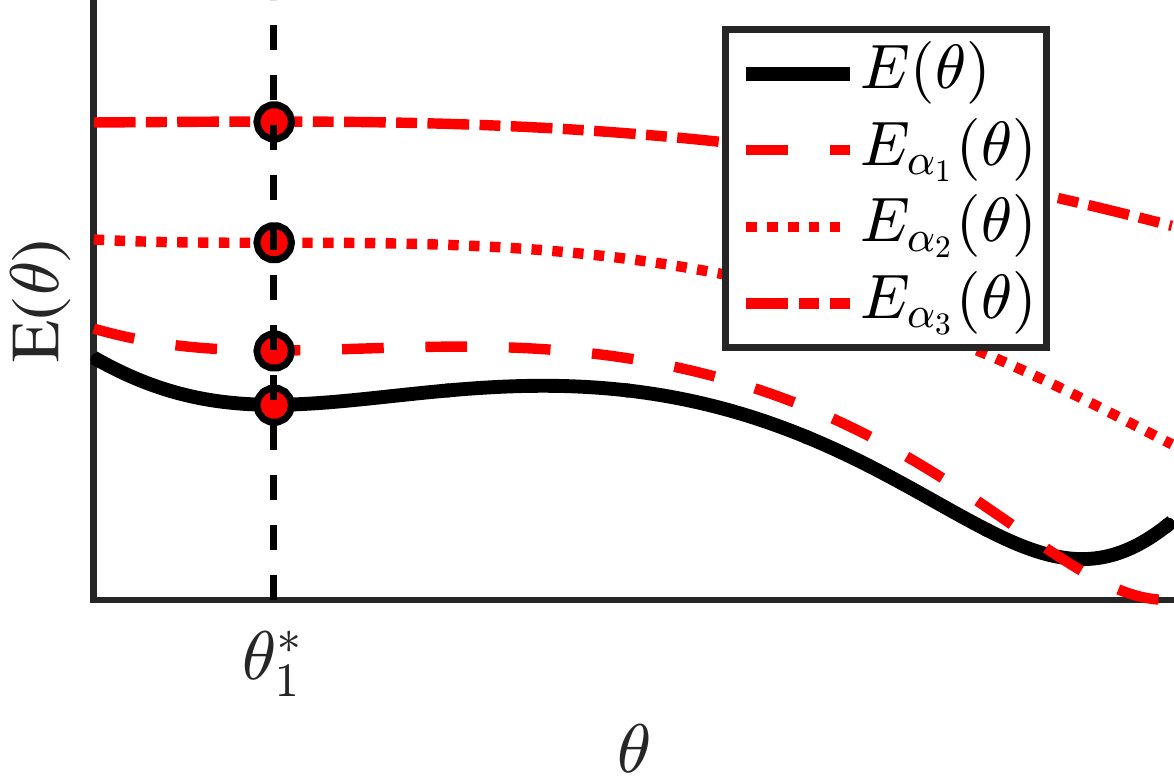}
			\label{fig:REConvergence1}} 
		\hspace{2mm}
		\centering
		\subfloat[$\theta_2^*$ and its expanded objective function.]{\includegraphics[width=0.45\linewidth]{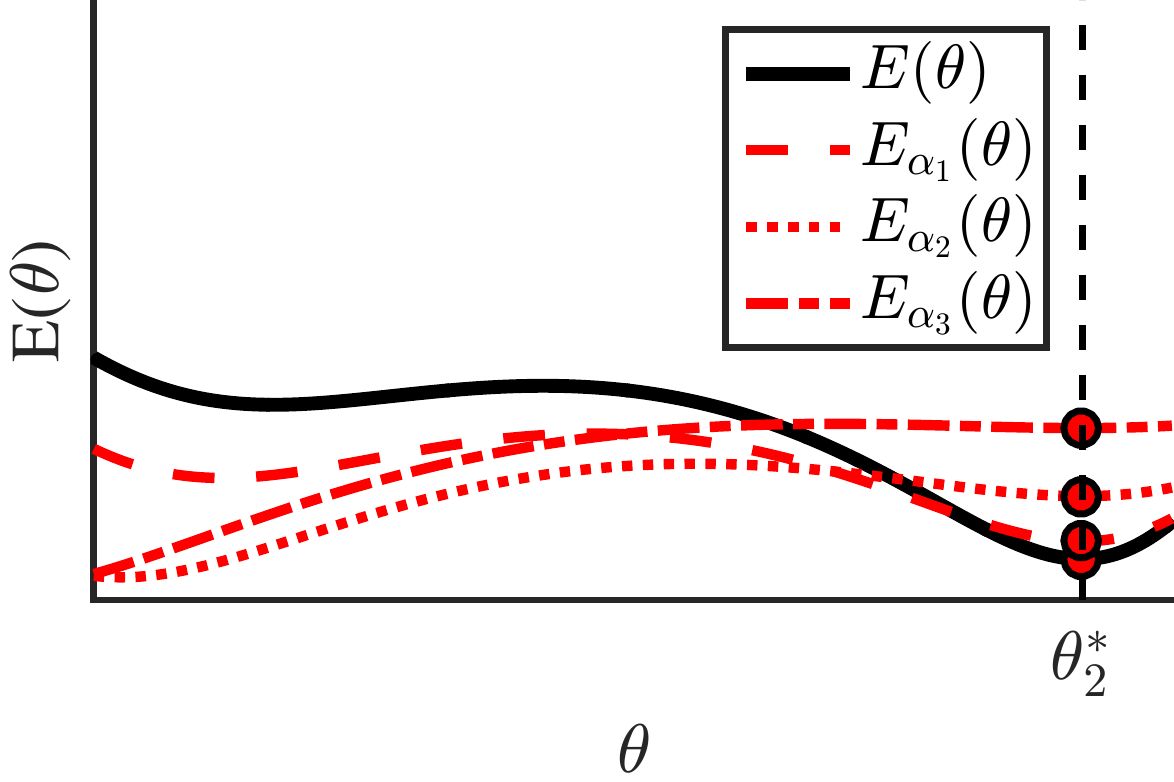}
			\label{fig:REConvergence2}}
		%		\end{center}
		\vspace{-2mm}
		\caption{Expanded objective functions $E_\alpha(\theta)$ with different local minima, $\theta_1^*$ and $\theta_2^*$. Red broken lines denote different $\alpha$-expanded objective functions $E_\alpha(\theta)$ with $\alpha_1<\alpha_2<\alpha_3$. $\theta_2^*$ is still a local minimum of $E_{\alpha_3}(\theta)$, while $\theta_1^*$ is not. }
		\label{fig:REConvergence}
		\vspace{-2mm}
	\end{figure}
	
	{\bf One-dimensional quartic minimization:}
	Here we consider a quartic minimization problem---in particular, one that can be formulated as an LS problem:
	\begin{equation}
		E\left(\theta\right) = \frac{1}{2}\left(\left(y_1-\theta^2\right)^2+\left(y_2-\theta\right)^2\right).
		\label{eq:QuarticMinimization}
	\end{equation}
	We consider the case where \Eref{eq:QuarticMinimization} has two local minima $\theta_1$ and $\theta_2$. The following theorem then holds:
	\begin{theorem}
		Let $\theta_1$ and $\theta_2$ be local minima points of \Eref{eq:QuarticMinimization} with RE constants of $\alpha_1$ and $\alpha_2$, respectively. $\theta_1$ is the global minimum point if $\alpha_1 > \alpha_2$ and $\theta_2$ is the global minimum point otherwise.
	\end{theorem}
	\begin{proof}
		Please refer to the supplementary materials.
	\end{proof} 
	
	\subsection{General relationship between the $\alpha$ RE convergence and the global optimum}
	\begin{figure}[t]
		\begin{center}
			\subfloat[K-means clustering result of a poor local minimum.]{\includegraphics[width=0.45\linewidth]{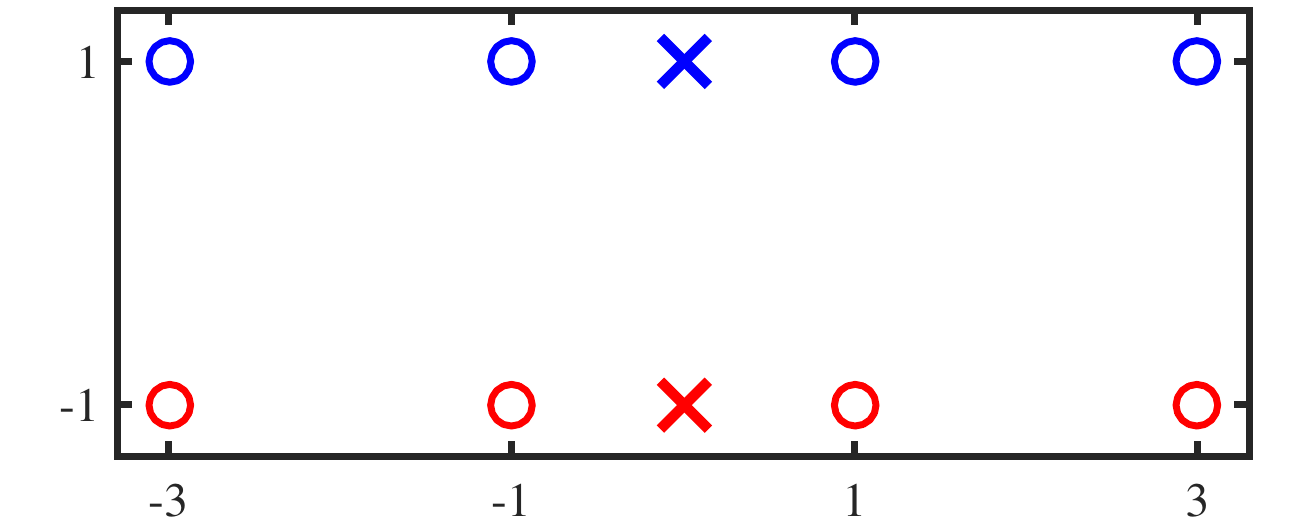}
				\label{fig:CounterExample1}}
			\hspace{2mm}
			\subfloat[K-means clustering result of a global minimum.]{\includegraphics[width=0.45\linewidth]{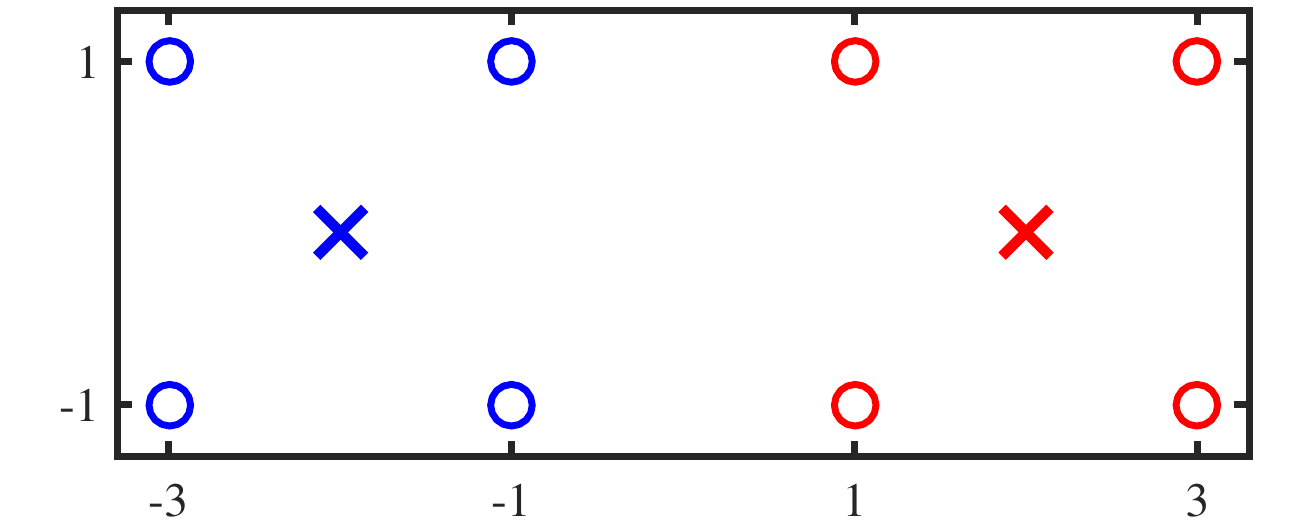}
				\label{fig:CounterExample2}}
		\end{center}
		\vspace{-3mm}
		\caption{Different k-means clustering results with $k=2$. The result (a) has an RE constant of $\alpha=\infty$; however, this is a poor local minimum. On the other hand, the result (b) has finite RE constant; however, this is a global minimum. }
		\label{fig:CounterExample}
		\vspace{-0mm}
	\end{figure}
	
	It is not obvious when our hypothesis, i.e., that a solution with a larger RE constant is closer to the global optimum, is valid. Unfortunately, we can easily find a counterexample in k-means clustering, as shown in \Fref{fig:CounterExample}. However, our algorithm, which aims to find a solution with a large RE constant, works well from an empirical perspective in many nonconvex LS problems.

	\section{Residual expansion algorithm}
	\label{sec:REalgorithm}
	In this section, we propose the RE algorithm, which aims to find a solution with a large RE constant.
	Since it is difficult to find the solution with the largest RE constant exactly, we employ a heuristic strategy.
	\begin{figure}[t]
		\begin{center}
			\includegraphics[width=0.99\linewidth]{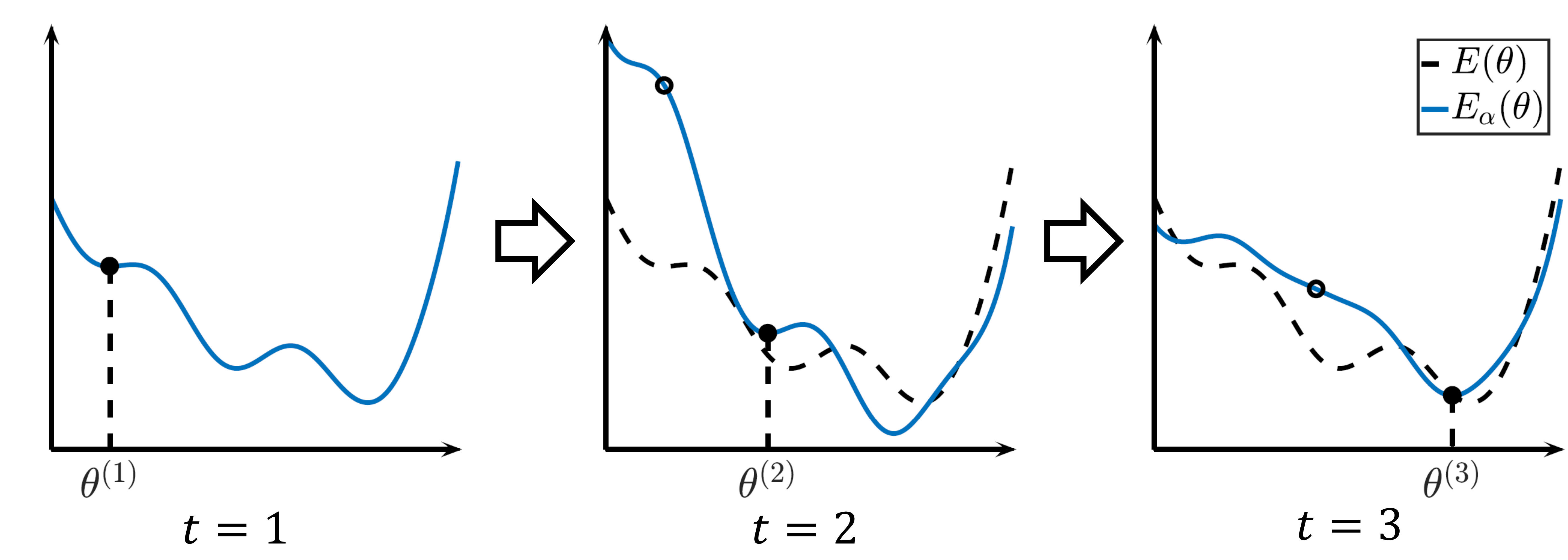}
		\end{center}
		\vspace{-5mm}
		\caption{Conceptual view of the RE algorithm. The algorithm iterates parameter updating and residual expansion, which elevates the objective function for the current solution.}
		\label{fig:REalgorithm}
		\vspace{-3mm}
	\end{figure}

	The RE algorithm has two steps: parameter updating and residual expansion. We show an intuitive illustration of the algorithm in \Fref{fig:REalgorithm}. For the residual expansion step, we expand data along their residual direction. This results in elevating the objective function around the current solution as in \Fref{fig:REConvergence}. For the parameter-updating step, instead of minimizing the original function \Eref{eq:LS}, we minimize the following expanded objective function in each iteration:
	\begin{equation}
		E_t\left(\bvec{\uptheta}\right) = \frac{1}{2}\lVert\hat{\bvec{y}}^{(t)}
		-\bvec{f}\left(\bvec{\uptheta}\right)\rVert_2^2,
		\label{eq:REalgorithm1}
	\end{equation}
	where $\hat{\bvec{y}}^{(t)}$ is an expanded data vector:
	\begin{eqnarray}
		\hat{\bvec{y}}^{(t)}&=&\bvec{y}+\alpha^{(t)}\bvec{r}^{(t)},\label{eq:REalgorithm2_1} \\
		\bvec{r}^{(t)} &=& p^{(t)}(\bvec{y}-\bvec{f}(\bvec{\uptheta}^{(t)}))+(1-p^{(t)})\bvec{r}^{(t-1)}.
		\label{eq:REalgorithm2_2}
	\end{eqnarray}
	where $\alpha$ and $0<p\leq 1$ are expansion and momentum parameters, respectively. Note that, if $p=1$, \Eref{eq:REalgorithm1} is an exactly $\alpha^{(t)}$-expanded objective function on $\bvec{\uptheta}^{(t)}$. The momentum parameter is important for achieving good performance and ensuring that the RE algorithm does not to diverge, as described later.
	
	The RE algorithm iterates through parameter updating by minimizing \Eref{eq:REalgorithm1} and residual expansions by \Eref{eq:REalgorithm2_1} and \Eref{eq:REalgorithm2_2}. We use a large $\alpha^{(0)}$ initially to find a solution with a large RE constant. Then we decrease $\alpha^{(t)}$ gradually to achieve convergence, analogous to a temperature parameter in SA. We summarize the RE algorithm in \Aref{alg:REalgorithm}.
\setlength{\textfloatsep}{5pt}
	\begin{algorithm}[t]
		\caption{Residual expansion algorithm.}
		\begin{algorithmic}[1]
			\Input{Expansion parameter $\alpha^{(t)}\rightarrow 0$, momentum $p^{(t)}$.}
			\Initialize{$t=0, \hat{\bvec{y}}^{(0)}=\bvec{y}, \bvec{r}^{(0)}=\bvec{0}$}
			%			\Statex
			\While{not converged}
			\State{Update $\bvec{\uptheta}$ by \Eref{eq:REalgorithm1} (or \Eref{eq:RegularizedLS2})}
			\State{$\bvec{r}^{(t+1)} = p^{(t)}\left(\bvec{y}-\bvec{f}\left(\bvec{\uptheta}^{(t+1)}\right)\right)+\left(1-p^{(t)}\right)\bvec{r}^{(t)} $}
			\State{$\hat{\bvec{y}}^{(t+1)} = \bvec{y}+ \alpha^{(t)}\bvec{r}^{(t+1)}$}
			\State{$t=t+1$}
			\EndWhile
			\Output{$\bvec{\uptheta}$}
		\end{algorithmic}
		\label{alg:REalgorithm}
	\end{algorithm}

	\subsection{Parameter setting for convergence}
	\label{sec:parameterSetting}
	The RE algorithm has two parameters, $\alpha$ and $p$, for each iteration. We decrease $\alpha^{(t)}$ to 0 for convergence. when $\alpha=0$, there is no residual expansion and RE algorithm is guaranteed to converge if the original LS problem has a convergence-guaranteed algorithm. However, inadequate parameters of $\alpha$ and $p$ cause unstable optimization. We consider the norm of $\bvec{r}^{(t+1)}$. We can obtain
%	The RE algorithm has two parameters, $\alpha$ and $p$, for each iteration. We decrease $\alpha^{(t)}$ to 0, and thus convergence is guaranteed. For proper values of $\alpha$ and $p$, we consider the norm of $\bvec{r}^{(t+1)}$; then we can obtain
	\begin{eqnarray}
		\left\lVert\bvec{r}^{(t+1)}\right\rVert_2^2 &=&\left\lVert p\left(\bvec{y}-\bvec{f}\left(\bvec{\uptheta}^{(t+1)}\right)\right) + \left(1-p\right)\bvec{r}^{(t)} \right\rVert_2^2 \nonumber \\ 
		&\sim& \left(1-p-\alpha p\right)^2\left\lVert\bvec{r}^{(t)}\right\rVert_2^2.
		\label{eq:parameterSetting}
	\end{eqnarray}
	We use $\bvec{f}\left(\bvec{\uptheta}^{(t+1)}\right)\sim \hat{\bvec{y}}^{(t)}$ for the last approximation of \Eref{eq:parameterSetting}. \Eref{eq:parameterSetting} suggests $\left(1-p-\alpha p\right)^2\leq1$ to make the RE algorithm stable. A good way to determine these values of $\alpha$ and $p$ is described in \Sref{sec:ADMM}.
	
	\subsection{Advantages of the RE algorithm}
	Our algorithm consists of two steps of parameter updating and residual expansion.
	Parameter updating is based simply on a typical LS problem approach.
	Therefore, if there is a source code which minimizes \Eref{eq:LS}, for example, by alternative optimization or gradient methods, then we can implement our algorithm by adding a residual expansion step to the existing code, which can be done in a few lines of code.
	
	Moreover, the computational complexity of residual expansion is generally less than that of parameter updating. Therefore, we can achieve faster optimization than most nonconvex optimization techniques based on multi-point search or random search, such as SA and GA. 
	
	As described in \Sref{sec:RelatedNonconvex}, GNC is a similar approach to ours; however GNC often does not apply for LS problems. Our algorithm can be applied for any nonconvex LS problem provided that there is a method for finding a local optimum, such as Lloyd's algorithm for k-means clustering and ICP algorithms for point-set registration.

	\section{Alternate direction method of multipliers for least squares problems}
	\label{sec:ADMM}
	In this section, we apply the alternate direction method of multipliers (ADMM)~\cite{boyd2011distributed} to solve \Eref{eq:LS}. We show that ADMM is a special case of the RE algorithm for LS problems. Moreover, ADMM suggests a modified RE algorithm for regularized LS problems.
	
	We introduce an auxiliary variable $\bvec{z} = \bvec{y}-\bvec{f}\left(\bvec{\uptheta}\right)$ and rewrite \Eref{eq:LS} as a constrained optimization problem:
	\begin{align}
		&\min_{\bvec{z},\bvec{\uptheta}}\frac{1}{2}\lVert\bvec{z}\rVert_2^2 \label{eq:LS_constrained} \\ \mbox{s.t.}\;\;&\bvec{z}=\bvec{y}-\bvec{f}\left(\bvec{\uptheta}\right). \nonumber
	\end{align}
	We can construct the augmented Lagrangian function of \Eref{eq:LS_constrained} as
	{
	\thickmuskip=0mu
	\medmuskip=0mu
	\thinmuskip=0mu
	\begin{equation}
		L_t(\bvec{\uptheta},\bvec{z},\uplambda) = \frac{1}{2}\lVert\bvec{z}\rVert_2^2 + \uplambda^\top\left(\bvec{z}-\bvec{y}+\bvec{f}\left(\bvec{\uptheta}\right)\right) + \frac{\mu^{(t)}}{2}\lVert\bvec{z}-\bvec{y}+\bvec{f}\left(\bvec{\uptheta}\right)\rVert_2^2
		\label{eq:LR_LagrangeFunction}.
	\end{equation}
	}
	We take the alternating direction approach for solving \Eref{eq:LR_LagrangeFunction} and then obtain update rules as
	\begin{eqnarray}
		&\bvec{\uptheta}^{(t+1)} = \argmin_{\bvec{\uptheta}} L_t(\bvec{\uptheta}, \bvec{z}^{(t)}, \uplambda^{(t)}) \label{eq:ADMM_update1_1} \\
		&\bvec{z}^{(t+1)} = \argmin_{\bvec{z}} L_t(\bvec{\uptheta}^{(t+1)}, \bvec{z}, \bvec{\uplambda}^{(t)}) \label{eq:ADMM_update1_2}\\
		&\bvec{\lambda}^{(t+1)}=\bvec{\lambda}^{(t)}+\mu^{(t)}\left(\bvec{z}^{(t+1)}-\bvec{y}+\bvec{f}\left(\bvec{\uptheta}\right)^{(t+1)}\right) \label{eq:ADMM_update1_3}
	\end{eqnarray}

	\subsection{Relation to the RE algorithm}
	
%	We can simplify \Eref{eq:ADMM_update1_1}, \Eref{eq:ADMM_update1_2} and \Eref{eq:ADMM_update1_3} by introducing a new variable $\bvec{r}^{(t)} = \bvec{z}^{(t)}-\uplambda^{(t)}/\mu^{(t)}$ as
	We can simplify \Eref{eq:ADMM_update1_1}, \Eref{eq:ADMM_update1_2} and \Eref{eq:ADMM_update1_3} as
	{\thickmuskip=0mu
	\medmuskip=0mu
	\thinmuskip=0mu
	\begin{eqnarray}
		\bvec{\uptheta}^{(t+1)}=\argmin_{\bvec{\uptheta}}{\frac{\mu^{(t)}}{2}\left\lVert\bvec{y}+\left(\frac{1-\mu^{(t)}}{\mu^{(t)}}\right)\bvec{z}^{(t)}-\bvec{f}\left(\bvec{\uptheta}\right)\right\rVert_2^2,\label{eq:ADMM_update2_1}} \\
		\bvec{z}^{(t+1)}=\left(\frac{1}{1+\mu^{(t)}}\right)\bvec{z}^{(t)}+\left(\frac{\mu^{(t)}}{1+\mu^{(t)}}\right)\left(\bvec{y}-\bvec{f}\left(\bvec{\uptheta}^{(t+1)}\right)\right).
		\label{eq:ADMM_update2_2}
	\end{eqnarray}}Details of the derivation are described in the supplementary material. This is a special case of the RE algorithm of \Eref{eq:REalgorithm2_1} and \Eref{eq:REalgorithm2_2} with
	\begin{eqnarray}
		\alpha^{(t)} = (1-\mu^{(t)})/\mu^{(t)} \label{eq:relationADMMandRE_1}, \\ p^{(t)} = \mu^{(t)}/(1+\mu^{(t)}). \label{eq:relationADMMandRE_2}
	\end{eqnarray}
	
	There are two main advantages to using ADMM. First, we can choose only $\mu$ instead of parameters $\alpha$ and $p$ in the general RE algorithm. \Eref{eq:relationADMMandRE_1} and \Eref{eq:relationADMMandRE_2} always satisfy $(1-p-\alpha p)^2<1$, which is a condition necessary for avoiding divergence to infinity, as described in \Sref{sec:parameterSetting}, and this update achieves good performance in experiments.
	Second, ADMM can treat regularized LS optimization problems, such as blind image deblurring (\Eref{eq:BlindImageDeconv}). We will describe this in the next section.

	\subsection{Regularized least squares problems}
	We consider a regularized LS problem as follows:
	\begin{equation}
		E\left(\bvec{\uptheta}\right) = \frac{1}{2}\lVert\bvec{y}-\bvec{f}\left(\bvec{\uptheta}\right)\rVert_2^2 + \gamma R\left(\bvec{\uptheta}\right).
		\label{eq:RegularizedLS}
	\end{equation}
	When we apply the RE algorithm in a straightforward manner, we can obtain the following objective function in each iteration:
	\begin{equation}
		E_t\left(\bvec{\uptheta}\right) = \frac{1}{2}\lVert\hat{\bvec{y}}^{(t)}-\bvec{f}\left(\bvec{\uptheta}\right)\rVert_2^2+ \gamma R\left(\bvec{\uptheta}\right).
		\label{eq:RegularizedLS2}
	\end{equation}
	In the case of ADMM, from \Eref{eq:ADMM_update2_1}, the objective function is as follows:
	\begin{equation}
		E_t\left(\bvec{\uptheta}\right) = \frac{\mu^{(t)}}{2}\lVert\hat{\bvec{y}}^{(t)}-\bvec{f}\left(\bvec{\uptheta}\right)\rVert_2^2+ \gamma R\left(\bvec{\uptheta}\right).
		\label{eq:RegularizedLS3}
	\end{equation}

	\begin{algorithm}[t]
		\caption{RE algorithm based on ADMM.}
		\begin{algorithmic}[1]
			\Input{Penalty parameter $\mu^{(t)}\rightarrow1$.}
			\Initialize{$t=0, \hat{\bvec{y}}^{(0)}=\bvec{y}, \bvec{r}^{(0)}=\bvec{0}$. }
			%			\Statex
			\While{not converged}
			\State{Update $\bvec{\uptheta}$ by \Eref{eq:RegularizedLS3}.}
			\State{$\bvec{r}^{(t+1)} = \mathrel{\left(\frac{\mu^{(t)}}{1+\mu^{(t)}}\right)\left(\bvec{y}-\bvec{f}\left(\bvec{\uptheta}^{(t+1)}\right)\right)+\left(\frac{1}{1+\mu^{(t)}}\right)\bvec{r}^{(t)} }$}
			\State{$\hat{\bvec{y}}^{(t+1)} = \bvec{y}+ \left(\frac{1-\mu^{(t)}}{\mu^{(t)}}\right)\bvec{r}^{(t+1)}$}
			\State{$t=t+1$}
			\EndWhile
			\Output{$\bvec{\uptheta}$}
		\end{algorithmic}
		\label{alg:REalgorithm2}
	\end{algorithm}

	We can find that the difference between \Eref{eq:RegularizedLS2} and \Eref{eq:RegularizedLS3} is simply the coefficient of the squared term. We find that minimizing \Eref{eq:RegularizedLS3} achieves better performance than minimizing \Eref{eq:RegularizedLS2}. 
	We summarize the RE algorithm based on ADMM in \Aref{alg:REalgorithm2}.

	\section{Implementation details}
	We used the RE algorithm based on ADMM (\Aref{alg:REalgorithm2}) unless otherwise stated. In the update of $\bvec{\uptheta}$ (line 2 in \Aref{alg:REalgorithm2}), we perform alternating optimization with a single iteration; for example, with k-means clustering, the cluster centers and assignments are updated only once. The four problems we treat in this paper can be minimized by alternating optimization.
	
	For the parameter $\mu$, we adapt $\mu^{(t+1)}=\min(\rho\mu^{(t)},1)$, where $\rho=\exp(-\log(\mu^{(0)})/T)$ to satisfy $\mu^{(T)}=1$. Therefore, we only need to determine the two parameters $\mu^{(0)}$ and $T$ in our method. 
	
	\begin{table*}[!t]
		\begin{center}
			\caption{Clustering results on synthetic data. Mean relative errors of the RE algorithm with different $\mu^{(0)}$ and $N$ are reported.}
			\label{tbl:KmeansSynthetic1}
			\subfloat[Synthetic data A with $k=100$.]{
				\scalebox{0.8}{
					\begin{tabular}{|l|c|c|c|c|} \hline
						& $\mu^{(0)}=0.5$ & $\mu^{(0)}=0.2$ & $\mu^{(0)}=0.1$ & $\mu^{(0)}=0.01$ \\ \hline
						$T=30$ & 0.905 & 0.894 & 0.902 & 0.921 \\ \hline
						$T=100$ & 0.854 & 0.856 & 0.862 & 0.876 \\ \hline
						$T=300$ & 0.843 & 0.844 & 0.846 & 0.854 \\ \hline
						$T=1,000$ & 0.837 & 0.839 & 0.840 & 0.843 \\ \hline
					\end{tabular}
				}
			}
			\subfloat[Synthetic data B with $k=10$.]{
				\scalebox{0.8}{
					\begin{tabular}{|l|c|c|c|c|} \hline
						& $\mu^{(0)}=0.5$ & $\mu^{(0)}=0.2$ & $\mu^{(0)}=0.1$ & $\mu^{(0)}=0.01$ \\ \hline
						$T=30$ & 2.699 & 1.758 & 1.493 & 0.998 \\ \hline
						$T=100$ & 2.784 & 1.209 & 0.789 & 0.630 \\ \hline
						$T=300$ & 2.722 & 1.036 & 0.708 & 0.552 \\ \hline
						$T=1,000$ & 2.749 & 0.994 & 0.630 & 0.552 \\ \hline
					\end{tabular}
				}
			}
		\end{center}
		\vspace{-7mm}
	\end{table*}
	
	\begin{table*}[!t]
		\begin{center}
			\caption{Clustering results on synthetic data. The mean / min / max relative errors and the average elapsed time are reported.}
			\label{tbl:KmeansSynthetic2}
			\subfloat[Synthetic data A with $k=100$.]{
				\scalebox{0.8}{
					\begin{tabular}{|l|c|c|c|c||c|} \hline
						\multicolumn{2}{|c|}{} & \multicolumn{3}{|c||}{Relative error} & \multirow{2}{14mm}{Elapsed time [sec]} \\ \cline{3-5}
						\multicolumn{2}{|c|}{} & Mean & Min & Max & \\ \hline
						\multicolumn{2}{|c|}{Random seeding} & 1.246 & 1.102 & 1.473 & 0.058 \\ \hline
						\multicolumn{2}{|c|}{k-means++~\cite{arthur2007k}} & 1.000 & 0.944 & 1.081 & 0.244 \\ \hline
						\multicolumn{2}{|c|}{Hartigan's algorithm~\cite{hartigan1975clustering}} & 0.925 & 0.881 & 0.981 & 0.359 \\ \hline
						\multirow{4}{18mm}{RE algorithm ($\mu^{(0)}=0.1$)} & $T=30$ & 0.902 & 0.875 & 0.942 & 0.258 \\ \cline{2-6}
						& $T=100$ & 0.862 & 0.846 & 0.873 & 0.780 \\ \cline{2-6}
						& $T=300$ & 0.846 & 0.836 & 0.856 & 2.29 \\ \cline{2-6}
						& $T=1,000$ & 0.840 & 0.831 & 0.850 & 7.61 \\ \hline 					
					\end{tabular}
				}
			}
			\subfloat[Synthetic data B with $k=10$.]{
				\scalebox{0.8}{
					\begin{tabular}{|l|c|c|c|c||c|} \hline
						\multicolumn{2}{|c|}{} & \multicolumn{3}{|c||}{relative error} & \multirow{2}{14mm}{elapsed time [sec]} \\ \cline{3-5}
						\multicolumn{2}{|c|}{} & mean & min & max & \\ \hline
						\multicolumn{2}{|c|}{Random seeding} & 4.277 & 0.552 & 22.680 & 0.0194 \\ \hline
						\multicolumn{2}{|c|}{k-means++~\cite{arthur2007k}} & 1.000 & 0.552 & 8.743 & 0.0271 \\ \hline
						\multicolumn{2}{|c|}{Hartigan's algorithm~\cite{hartigan1975clustering}} & 1.000 & 0.552 & 8.743 & 0.0429 \\ \hline
						\multirow{4}{18mm}{RE algorithm ($\mu^{(0)}=0.1$)} & $T=30$ & 1.493 & 0.552 & 6.777 & 0.0699 \\ \cline{2-6}
						& $T=100$ & 0.789 & 0.552 & 2.500 & 0.176 \\ \cline{2-6}
						& $T=300$ & 0.708 & 0.552 & 2.500 & 0.473 \\ \cline{2-6}
						& $T=1,000$ & 0.630 & 0.552 & 2.500 & 1.51 \\ \hline
					\end{tabular}
				}
			}
		\end{center}
		\vspace{-9mm}
	\end{table*}
	
	\section{Experimental results}
	We evaluate the performance of the RE algorithm on four nonconvex LS problems: k-means clustering, 3D point set registration, OPQ, and single blind image deblurring. All experiments were executed on an Intel Core i5-4200U CPU (1.60 GHz) with 8GB of RAM, and were implemented in MATLAB\footnote{Our codes will be available if the paper is accepted.} except for Go-ICP~\cite{yang2013go}. For Go-ICP and its comparison experiment, we used the publicly available code\footnote{http://iitlab.bit.edu.cn/mcislab/\~{}yangjiaolong/go-icp/} implemented in C++.

	\subsection{K-means clustering}
\setlength{\textfloatsep}{20pt}
	We compared our method with k-means++~\cite{arthur2007k}, which is a good initialization method, and Hartigan's algorithm~\cite{hartigan1975clustering}. For Hartigan's algorithm, we first used Lloyd's algorithm~\cite{lloyd1982least} with k-means++ initialization for fast computation. We reported the total time of Lloyd's algorithm and Hartigan's algorithm. For the other method, we used Lloyd's algorithm for optimization. We used random initialization for the RE algorithm.
	For error measurement, we used the objective function value of \Eref{eq:kmeans} and reported relative error from the value of k-means++ (therefore, the relative error of k-means++ is always 1). 
	
	\subsubsection{Synthetic data experiments}
	\label{sec:kmeansSynthetic}

	We start with two synthetic datasets as shown in \Fref{fig:KmeansSynthetic}. We repeated each method 50 times from different initializations and report the average relative errors. \Tref{tbl:KmeansSynthetic1} shows the results of our method with different $\mu^{(0)}$ and $T$. We found that larger $T$ achieved better performance. We also found that smaller $\mu^{(0)}$ achieved better performance in dataset B; however, larger $\mu^{(0)}$ achieved better performance in dataset~A. This indicates that the best setting $\mu^{(0)}$ is different for different data distributions. Intuitively, dataset B requires a larger residual expansion (in other words, small $\mu^{(0)}$) to escape from a poor local minimum, while dataset A requires a smaller residual expansion.
	
	\begin{figure}[t]
		\vspace{-2mm}
		\begin{center}
			\hfill
			\subfloat[Synthetic dataset A.]{\includegraphics[width=0.4\linewidth]{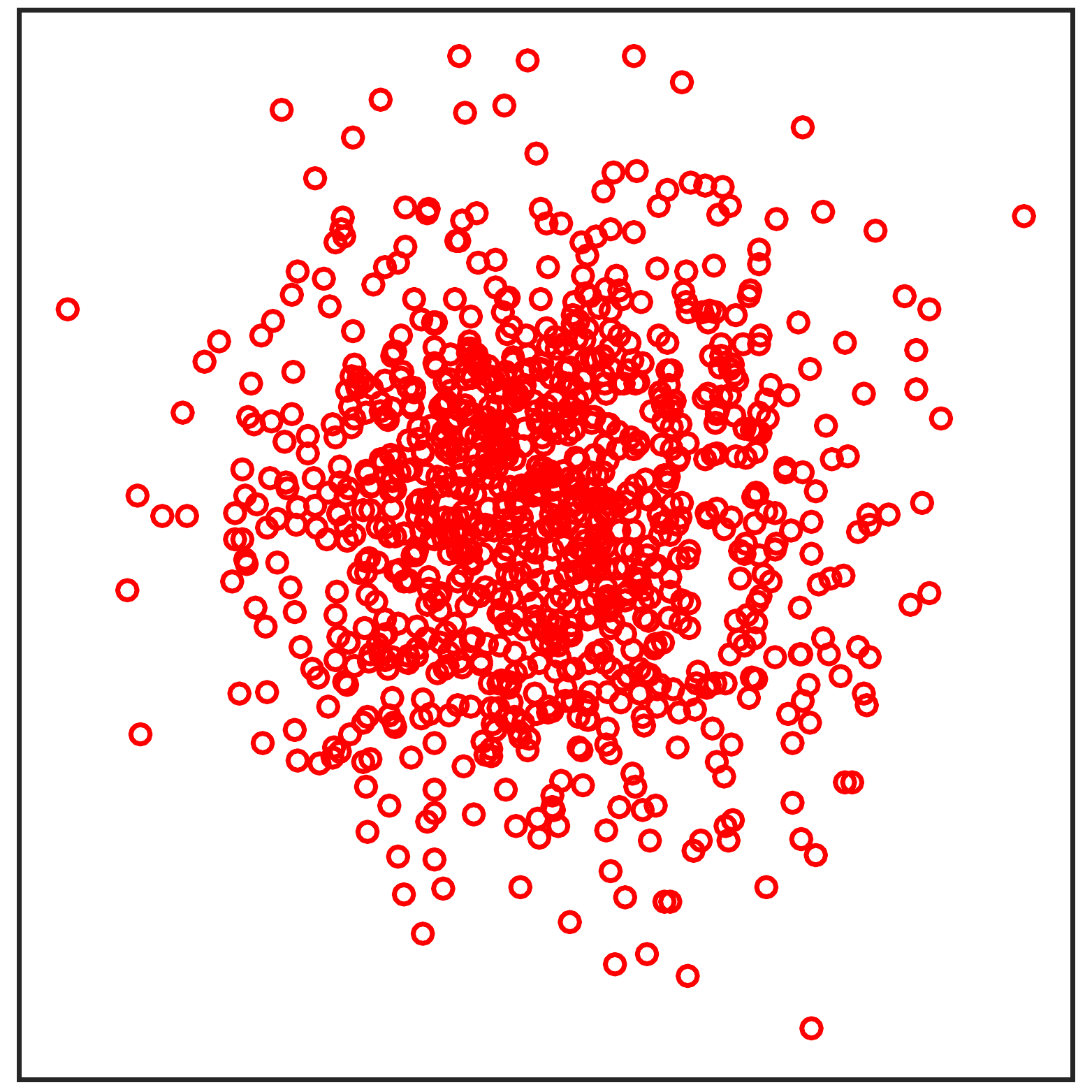}
				\label{fig:KmeansSynthetic1}} 
			\hfill
			\subfloat[Synthetic dataset B.]{\includegraphics[width=0.4\linewidth]{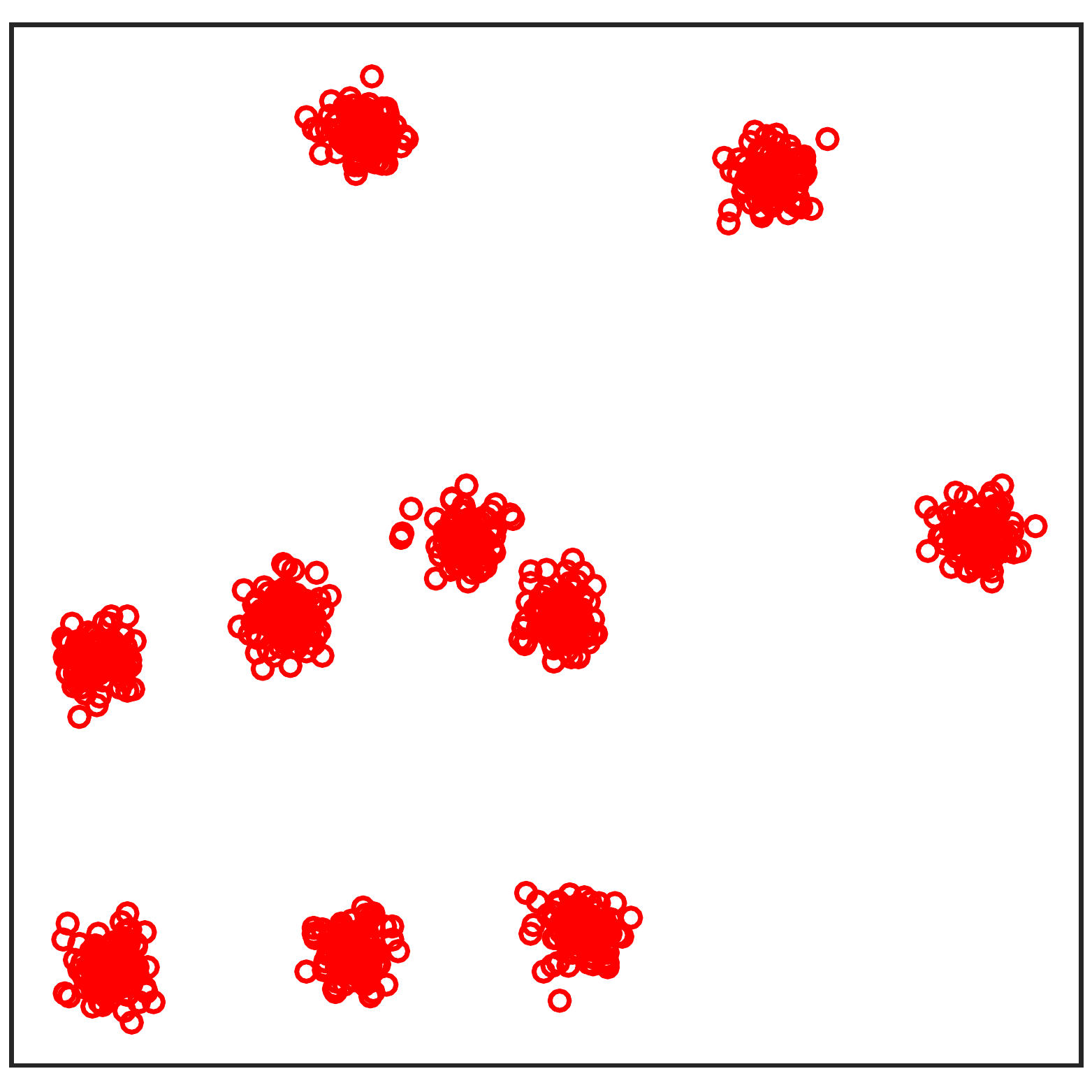}
				\label{fig:KmeansSynthetic2}} 
			\hfill\null
		\end{center}
		\vspace{-3mm}
		\caption{Synthetic data (1,000 two-dimensional points).}
		\label{fig:KmeansSynthetic}
		\vspace{-3mm}
	\end{figure}
	We show comparison results in \Tref{tbl:KmeansSynthetic2}. We repeated each method 50 times from different initializations. K-means++ worked well with dataset B. On the other hand, Hartigan's algorithm can improve the results of k-means++ in dataset A; however, this does not work in dataset B. The RE algorithm worked best in both cases, even though it was initialized by random seeding. Moreover, the RE algorithm with $T=30$ achieved comparable speed to k-means++ with better performance for dataset A.
	
	\subsubsection{Real-world data experiments}
		\begin{table*}[!t]
			\begin{center}
				\caption{Clustering results on real data. The mean / min / max of the relative error and the average elapsed time are reported.}
				\label{tbl:KmeansRealworld1}
				\vspace{-2mm}
				\subfloat[Cloud dataset ($\bmat{X}\in\mathbf{R}^{10\times1024}$) with $k=10$.]{
					\scalebox{0.8}{
						\begin{tabular}{|l|c|c|c|c||c|} \hline
							\multicolumn{2}{|c|}{} & \multicolumn{3}{|c||}{Relative error} & \multirow{2}{14mm}{Elapsed time [sec]} \\ \cline{3-5}
							\multicolumn{2}{|c|}{} & Mean & Min & Max & \\ \hline
							\multicolumn{2}{|c|}{Random seeding} & 1.255 & 1.003 & 1.438 & 0.0444 \\ \hline
							\multicolumn{2}{|c|}{k-means++~\cite{arthur2007k}} & 1.000 & 0.920 & 1.097 & 0.0395 \\ \hline
							\multicolumn{2}{|c|}{Hartigan's algorithm~\cite{hartigan1975clustering}} & 0.994 & 0.920 & 1.093 & 0.0593 \\ \hline
							\multirow{4}{18mm}{RE algorithm ($\mu^{(0)}=0.1$)} & $T=30$ & 0.980 & 0.920 & 1.031 & 0.0719 \\ \cline{2-6}
							& $T=100$ & 0.941 & 0.920 & 0.986 & 0.183 \\ \cline{2-6}
							& $T=300$ & 0.926 & 0.920 & 0.983 & 0.516 \\ \cline{2-6}
							& $T=1,000$ & 0.920 & 0.920 & 0.921 & 1.63 \\ \hline 					
						\end{tabular}
					}
				}
				\subfloat[COIL20 dataset ($\bmat{X}\in\mathbf{R}^{1024\times1440}$) with $k=20$.]{
					\vspace{-2mm}
					\scalebox{0.8}{
						\begin{tabular}{|l|c|c|c|c||c|} \hline
							\multicolumn{2}{|c|}{} & \multicolumn{3}{|c||}{Relative error} & \multirow{2}{14mm}{Elapsed time [sec]} \\ \cline{3-5}
							\multicolumn{2}{|c|}{} & Mean & Min & Max & \\ \hline
							\multicolumn{2}{|c|}{Random seeding} & 0.999 & 0.953 & 1.076 & 2.22 \\ \hline
							\multicolumn{2}{|c|}{k-means++~\cite{arthur2007k}} & 1.000 & 0.962 & 1.038 & 3.52 \\ \hline
							\multicolumn{2}{|c|}{Hartigan's algorithm~\cite{hartigan1975clustering}} & 0.990 & 0.960 & 1.021 & 8.07 \\ \hline
							\multirow{4}{18mm}{RE algorithm ($\mu^{(0)}=0.1$)} & $T=30$ & 0.951 & 0.939 & 0.977 & 4.37 \\ \cline{2-6}
							& $T=100$ & 0.945 & 0.938 & 0.960 & 12.6 \\ \cline{2-6}
							& $T=300$ & 0.942 & 0.938 & 0.950 & 36.6 \\ \cline{2-6}
							& $T=1,000$ & 0.941 & 0.938 & 0.956 & 119 \\ \hline
						\end{tabular}
					}
				}
			\end{center}
			\vspace{-10mm}
		\end{table*}
		
	We used two real-world datasets for comparison: the cloud dataset\footnote{https://archive.ics.uci.edu/ml/datasets/Cloud} and the COIL20 dataset~\cite{nene1996columbia}. We performed experiments in the same manner as in \Sref{sec:kmeansSynthetic}.
	%, and the SIFT-1M dataset~\cite{jegou2011product}
	
	\Tref{tbl:KmeansRealworld1} shows comparative results. In the cloud dataset, k-means++ achieves faster and better clustering than random seeding. The RE algorithm with $T=30$ achieved better clustering than k-means++ with about 1.8 times the computational cost. The RE algorithm with $T=1000$ performed best, and found the near-global optimum in every case. For the COIL20 dataset, although k-means++ and Hartigan's algorithm did not work well, the RE algorithm significantly outperformed the other methods.

	\subsection{Point set registration}
	\begin{table}[tbp]
		\begin{center}
			\caption{Point set registration results. We reported the number of successes with each different rotation angle. We also reported the average elapsed time for all 150 point sets.}
			\vspace{-2mm}
			\label{tbl:ICPResults1}
			\scalebox{0.7}{
				\begin{tabular}{|l|c|c|c|c||c|} \hline
					\multicolumn{2}{|c|}{} & \multicolumn{3}{|c||}{Number of successes} &\multirow{2}{16mm}{Number of iterations} \\ \cline{3-5}
					\multicolumn{2}{|c|}{} & $\phi=\pi/3$ & $\phi=5\pi/12$ & $\phi=\pi/2$ & \\ \hline
					\multicolumn{2}{|c|}{ICP algorithm} & 26 & 4 & 1 & 46.7 \\ \hline
					\multirow{4}{18mm}{RE algorithm ($\mu^{(0)}=0.1$)} & $T=30$ & 46 & 25 & 2 & 47.4 \\ \cline{2-6}
					& $T=100$ & 49 & 31 & 5 & 110.1 \\ \cline{2-6}
					& $T=300$ & 49 & 33& 6 & 310.2 \\ \cline{2-6}
					& $T=1,000$ & 49 & 36& 6 & 1008 \\ \hline
				\end{tabular}
			}	
		\end{center}
		\vspace{-7mm}
	\end{table}

	\begin{figure}[t]
		\begin{center}
			\hfill
			\subfloat[Source model (500 points). ]{\includegraphics[width=0.4\linewidth]{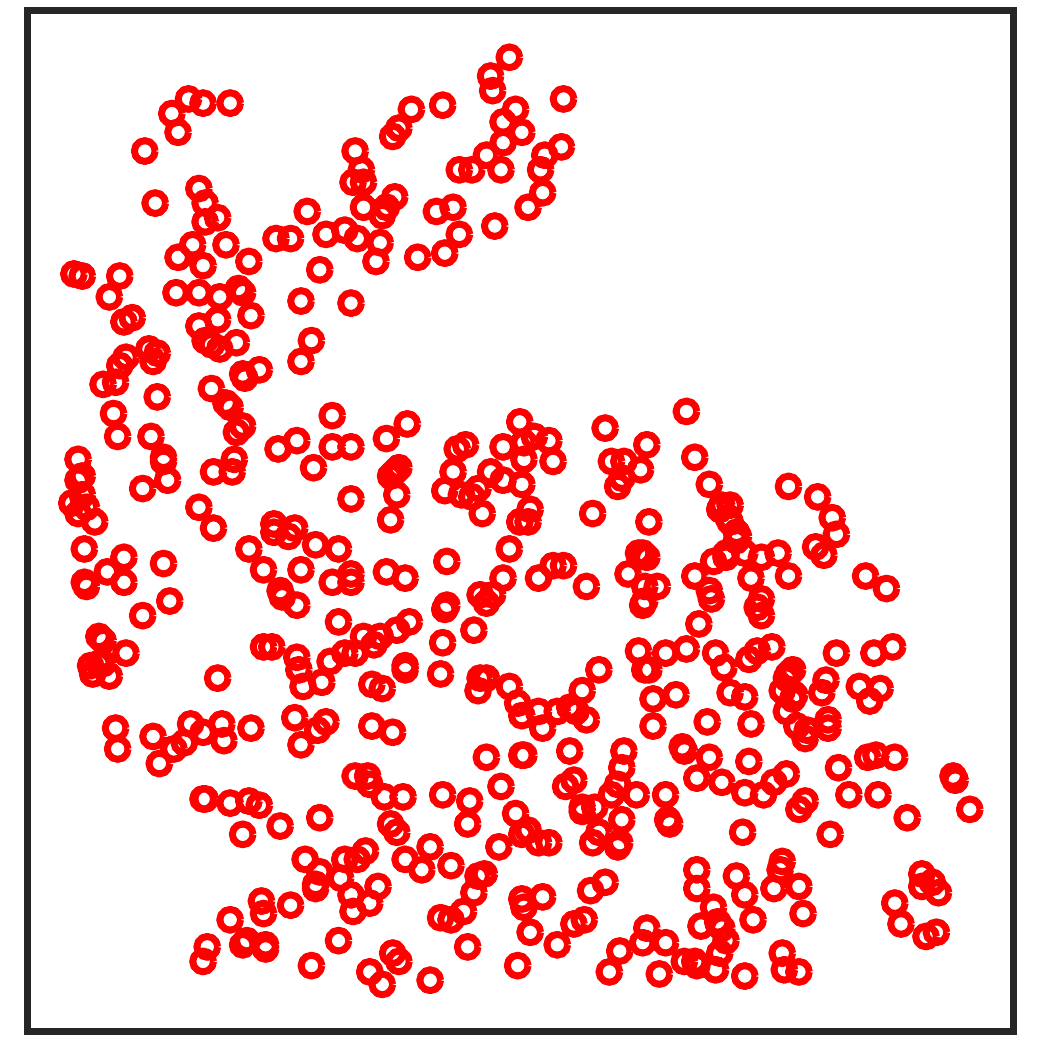}
				\label{fig:Bunny1}} 
			\hfill
			\subfloat[Target model (313 points).]{\includegraphics[width=0.4\linewidth]{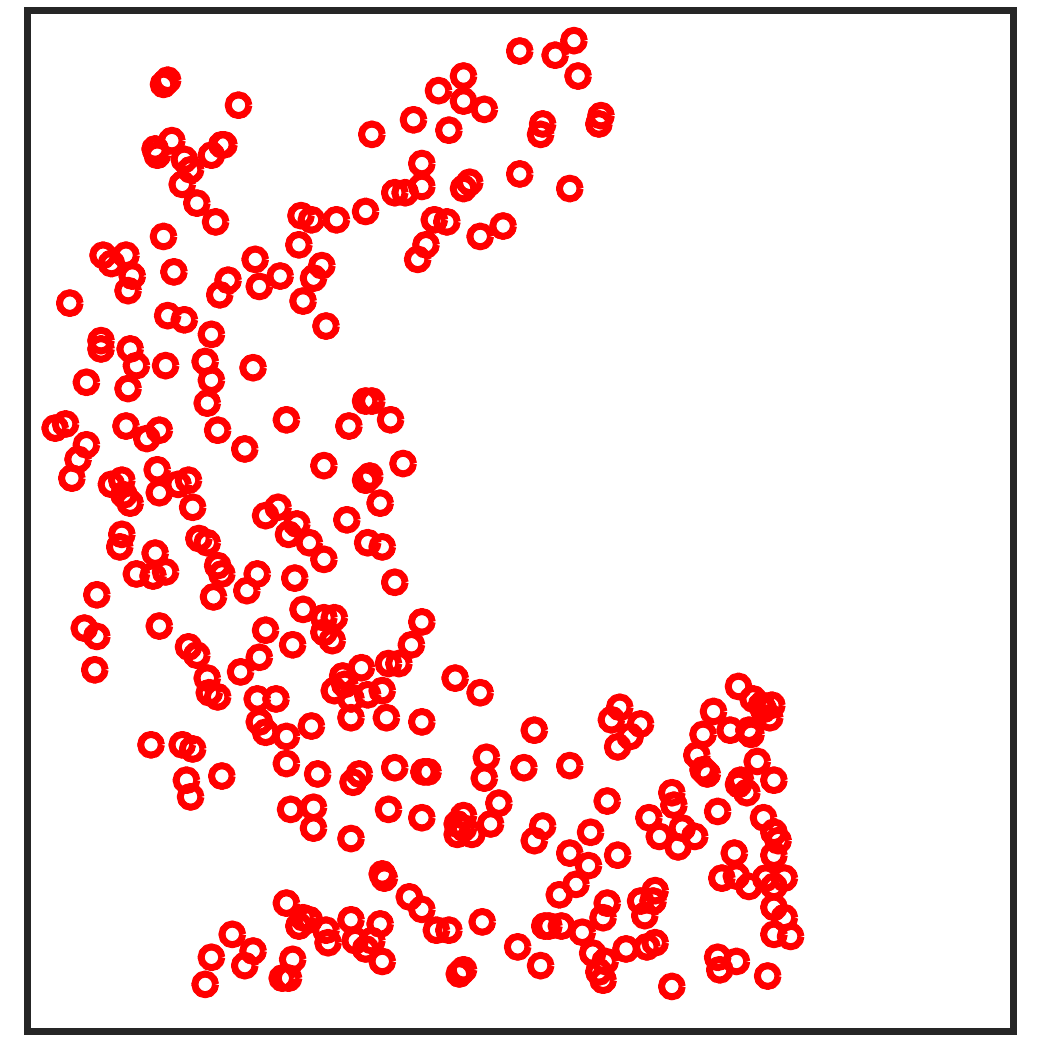}
				\label{fig:Bunny2}} 
			\hfill\null
		\end{center}
		\vspace{-5mm}
		\caption{Point set models.}
		\label{fig:Bunny}
		\vspace{-3mm}
	\end{figure}

	We compared the RE algorithm with the ICP algorithm and Go-ICP~\cite{yang2013go}. Go-ICP is known as a method that can achieve global optimization. We used the bunny model from the Stanford3D dataset\footnote{http://graphics.stanford.edu/data/3Dscanrep/}, as in \Fref{fig:Bunny}. For the target model, we used a partial point set as in \Fref{fig:Bunny2}. In the experiments, point sets were normalized within a cube of $[-1,1]^3$. 
	
	We made a rotation matrix $\bmat{R}_{gt}$ from a random rotation axis and the rotation angle $\phi$. The target point set was constructed by this rotation matrix, and we added Gaussian noise with a standard deviation of $\sigma=0.03$. We performed 50 tests with different random rotation axes at each rotation angle $\phi=\pi/3, 5\pi/12, \pi/2$. For measurement of the error, we used the objective value \Eref{eq:ICP} and regarded the results as successful if the objective error was less than 1 (this value is approximately twice of the average objective value using Go-ICP, as in \Fref{fig:GoICPResults}).
	
	We first show the comparison results between the RE algorithm and the ICP algorithm as in \Tref{tbl:ICPResults1}. The RE algorithm with $T=30$ achieved a better success rate with almost the same number of iterations as the ICP algorithm. Using a large $T$ can improve the results to a small extent.
	
	We also compare our method to Go-ICP~\cite{yang2013go}. Go-ICP has two steps: the ICP algorithm and the branch-and-bound algorithm. We compared the original Go-ICP and RE + Go-ICP, which has the two steps of ICP with the RE and branch-and-bound algorithms. \Fref{fig:GoICPResults} plots all 50 results in $\phi=5\pi/12$. Note that this comparison was implemented entirely in C++. Go-ICP always found the global optimum solution; however, it required significant computation. RE + Go-ICP reduced computational cost while achieving global optimization.
	
\begin{figure}[t]
	\centering
	\includegraphics[width=0.9\linewidth]{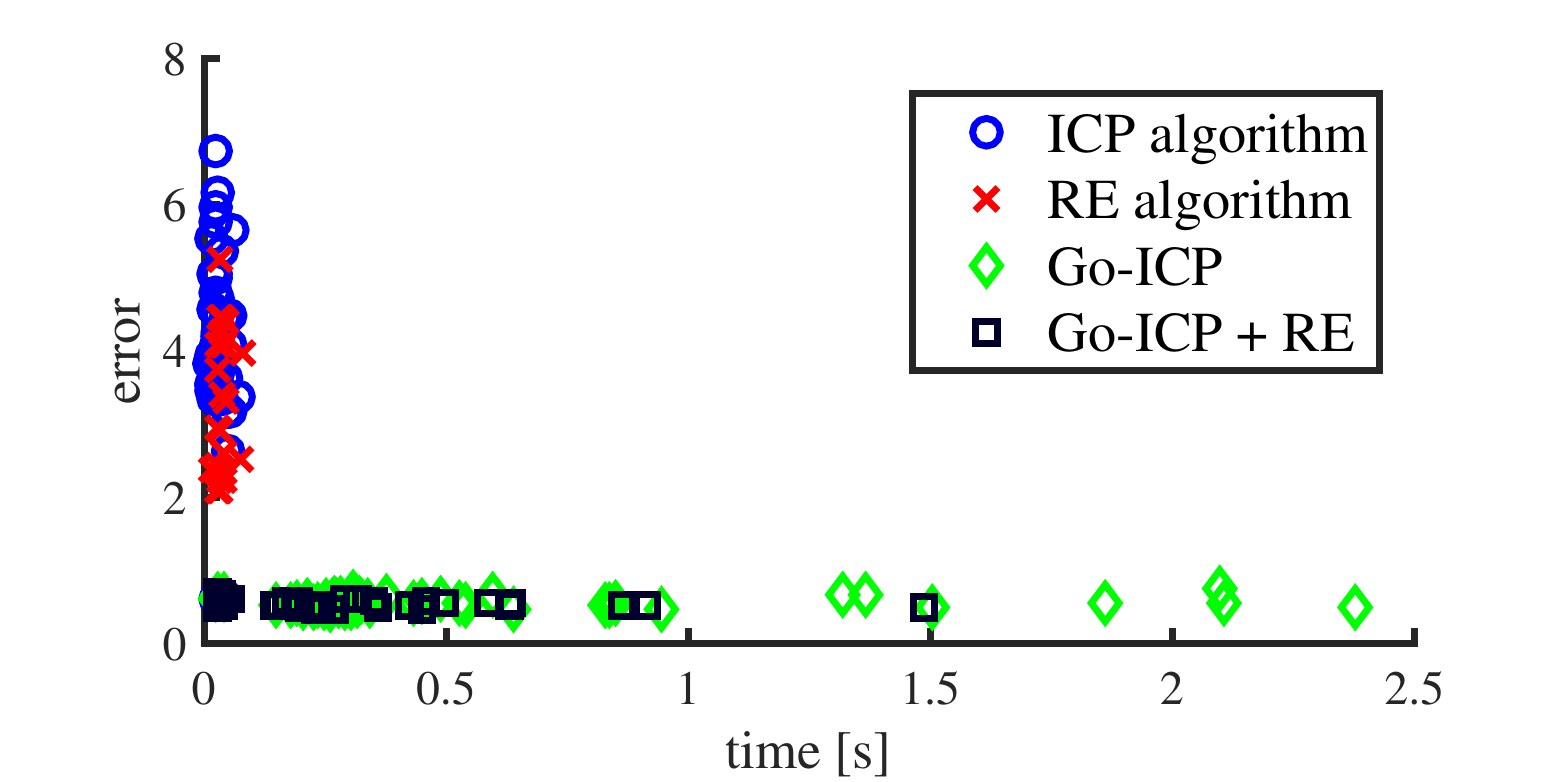}
	\vspace{-3mm}
	\caption{Point set registration results with $\phi=5\pi/12$. We plotted the objective value and computation time over 50 trials. For the RE algorithm, we set $T=30$. The average computational times for the ICP algorithm, RE algorithm, Go-ICP, and RE + Go-ICP were 0.0304, 0.0347, 0.551, and 0.231 seconds, respectively. }
	\label{fig:GoICPResults}
	\vspace{-3mm}
\end{figure}

\begin{table*}[tbp]
	\begin{center}
		\caption{Deblurring results. We reported PSNR values [dB] for each method. The best and second-best results are highlighted in bold and italics, respectively.}
		\label{tbl:DeblurResults1}
		\vspace{-2mm}
		\scalebox{0.75}{
			\begin{tabular}{l|cccc|cccc|cccc|cccc|cccc} \hline
				Image & \multicolumn{4}{|c|}{\#1} & \multicolumn{4}{|c|}{\#2} & \multicolumn{4}{|c|}{\#3} & \multicolumn{4}{|c|}{\#4} & \multicolumn{4}{|c}{\#5} \\ \hline
				Kernel & \#1 & \#2 & \#3 & \#4 & \#1 & \#2 & \#3 & \#4 & \#1 & \#2 & \#3 & \#4 & \#1 & \#2 & \#3 & \#4 & \#1 & \#2 & \#3 & \#4 \\ \hline\hline
				Pan \etal~\cite{pan2014deblurring} & 13.6 & 13.2 & 20.3 & 13.1 & 14.9 & 16.5 & 14.1 & {\bf 11.2} & 11.8 & {\it 20.6} & 19.4 & {\it 10.0} & 13.7 & 11.8 & {\it 19.1} & 11.9 & 19.6 & {\bf 19.2} & 16.8 & {\it 14.1} \\ \hline
				\Aref{alg:REalgorithm} & {\bf 20.2} & {\bf 20.2} & {\bf 21.3} & {\bf 16.4} & {\it 17.0} & {\it 16.9} & {\bf 15.8} & 7.4 & {\bf 20.9} & 18.8 & {\it 20.4} & 6.5 & {\bf 23.6} & {\it 20.4} & {\bf 19.7} & {\bf 12.5} & {\it 21.3} & {\it 18.2} & {\it 17.5} & 9.4 \\ \hline
				\Aref{alg:REalgorithm2} & {\bf 20.2} & {\it 19.6} & {\it 20.5} & {\it 15.7} & {\bf 17.2} & {\bf 17.1} & {\bf 15.8} & {\it 8.4} & {\it 19.5} & {\bf 21.5} & {\bf 20.7} & {\bf 14.2} & {\bf 23.6} & {\bf 20.5} & {\it 19.1} & {\it 11.9} & {\bf 21.9} & 17.7 & {\bf 18.6} & {\bf 15.2} \\ \hline
			\end{tabular}
		}	
	\end{center}
	\vspace{-9mm}
\end{table*}

\setcounter{figure}{8}  
\begin{figure*}[t]
	\begin{center}
		\subfloat[Latent image.]{
			\includegraphics[width=0.24\linewidth]{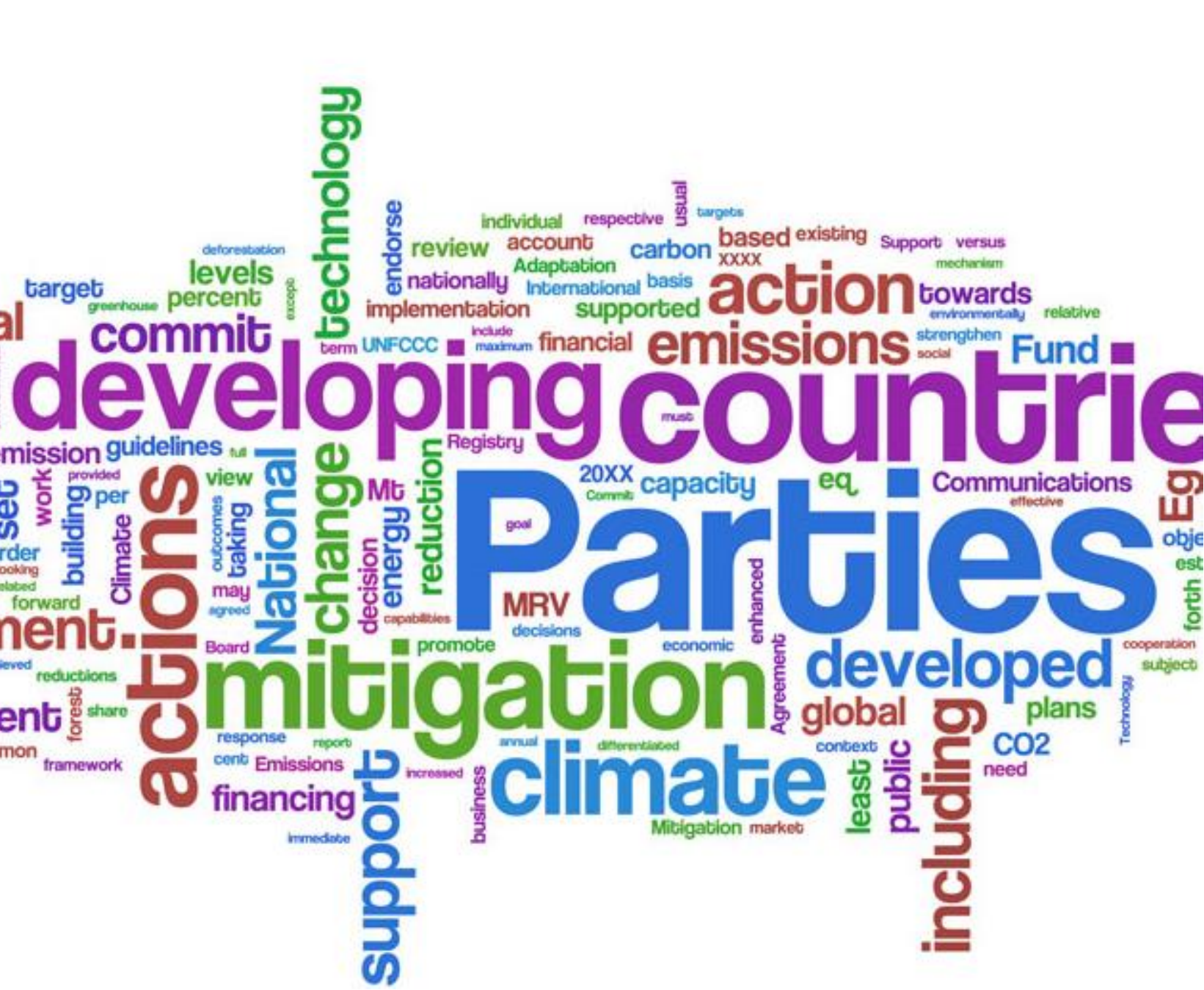}}
		\subfloat[Blurred image.]{
			\includegraphics[width=0.24\linewidth]{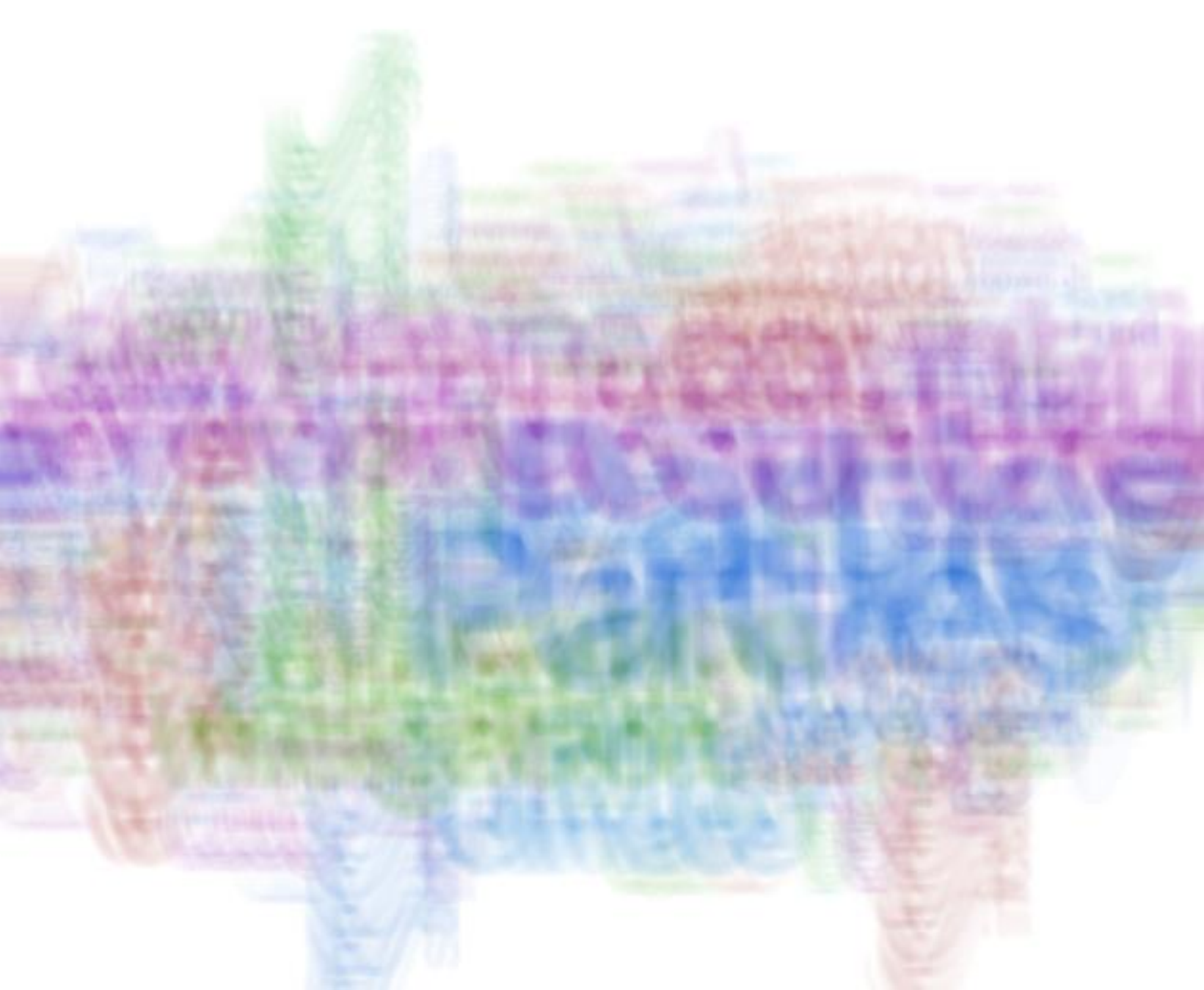}}
		\subfloat[Result using \cite{pan2014deblurring}.]{
			\includegraphics[width=0.24\linewidth]{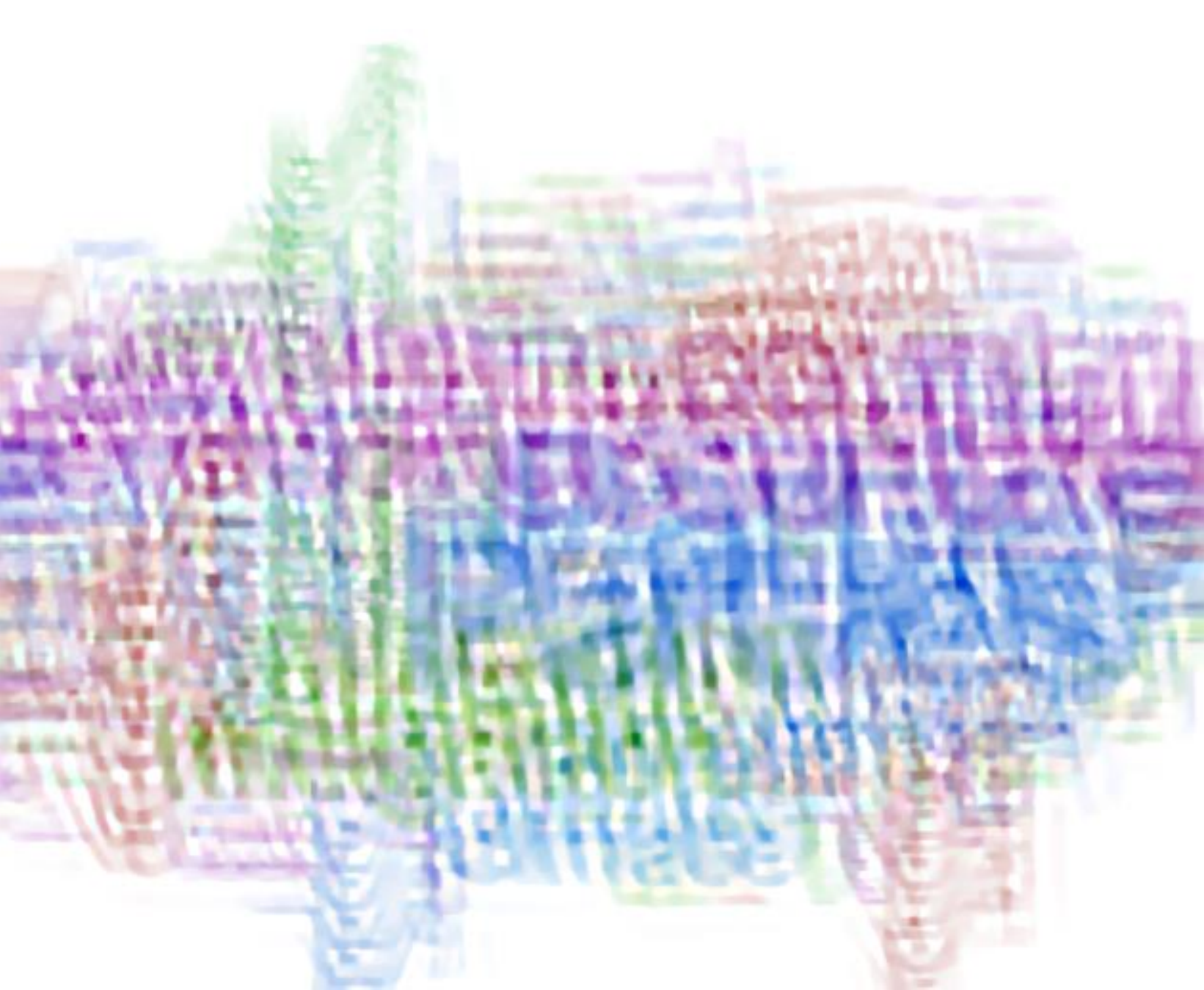}}
		\subfloat[Result using \Aref{alg:REalgorithm2}.]{
			\includegraphics[width=0.24\linewidth]{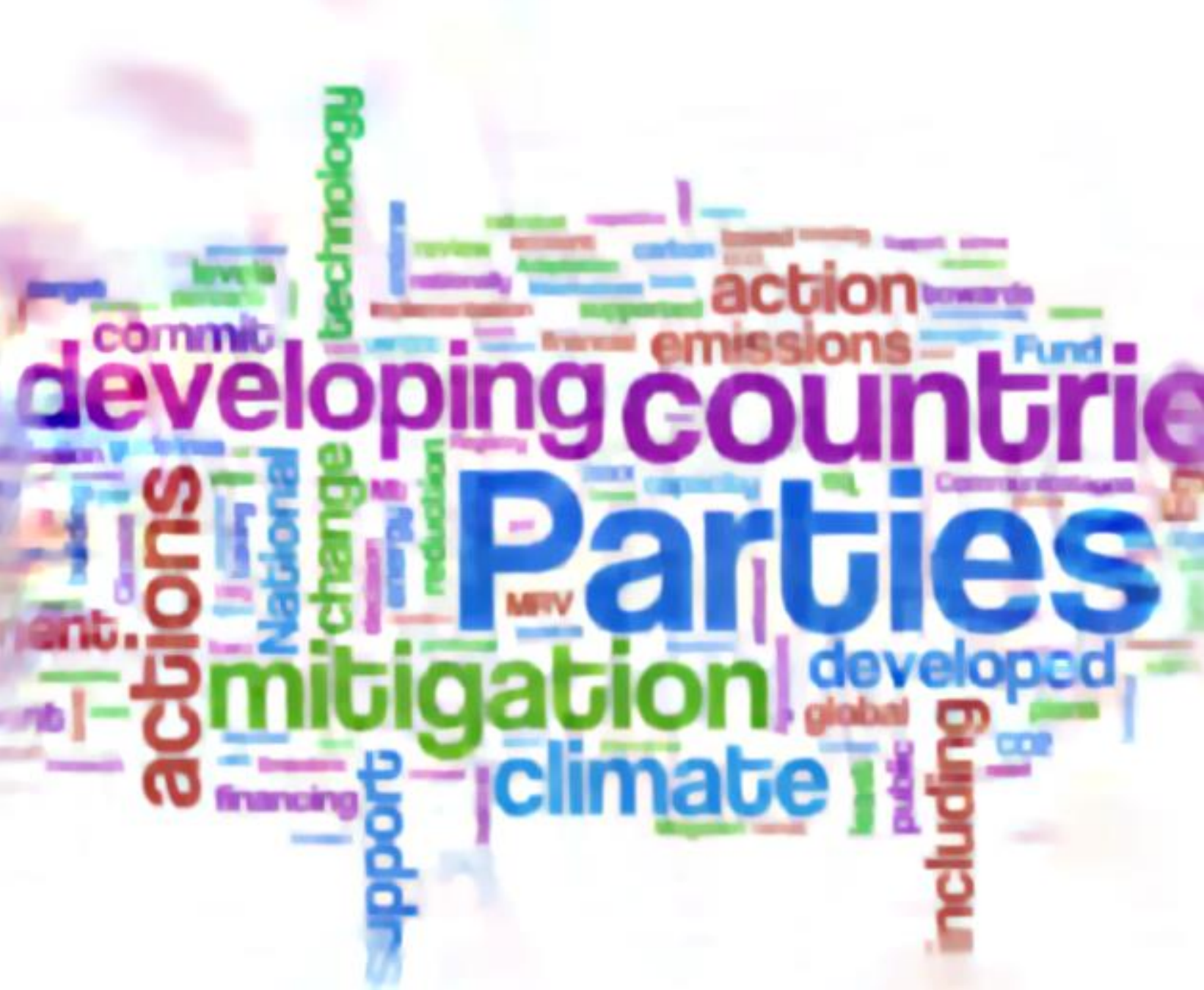}}		
	\end{center}
	\vspace{-4mm}
	\caption{An example of the results of deblurring (image \#1, kernel \#4).}
	\vspace{-3mm}
	\label{fig:DeblurFig}
\end{figure*}

	\subsection{Optimized product quantization}
	We show that the RE algorithm is successful in OPQ optimization problems. We compare our method with the alternating optimization method~\cite{ge2013optimized,norouzi2013cartesian}. For a dataset, we used SIFT 1M~\cite{jegou2011product}, which contains 100,000 128-dimensional SIFT descriptors for training. We set the subspace number $M=8$ and cluster number $k=256$, which are often used in the field of approximate nearest neighbor search. For error measurement, we used the objective function value of \Eref{eq:OPQ}. For our method, we set $\mu^{(0)}=0.5$.

	We plot the objective function value versus iteration number in \Fref{fig:OPQObjVsIter}. We performed five repetitions using different initializations and report the average objective values obtained. Our method improved the objective function value; moreover, it achieved rapid convergence in the cases of $T=30$ and $T=100$. The RE algorithm elevates the objective function around the current solution; in other words, it transforms the gradient for the current solution into a steeper gradient, potentially causing rapid convergence.

	\subsection{Blind image deblurring}	
	We evaluated our method with single blind image deblurring. There are many formulations for blind image deblurring. In this paper, we followed Pan \etal's formulation~\cite{pan2014deblurring}, which can be minimized by alternating optimization. We compared three methods as follows: a coarse-to-fine strategy~\cite{pan2014deblurring} and RE algorithms based on \Aref{alg:REalgorithm} and \Aref{alg:REalgorithm2}. We used the uniform blurred text images from the dataset provided by Lai \etal~\cite{Lai_2016_CVPR}, which contains five latent images and four blurring kernels for a total of 20 blurred images. For all methods, we used the same objective function parameters, such as the regularization coefficients. For our method, we set $\mu^{(0)} = 0.2$ and $T=100$.
	\setcounter{figure}{7}  
			\begin{figure}[t]
				\centering
				\includegraphics[width=1.0\linewidth]{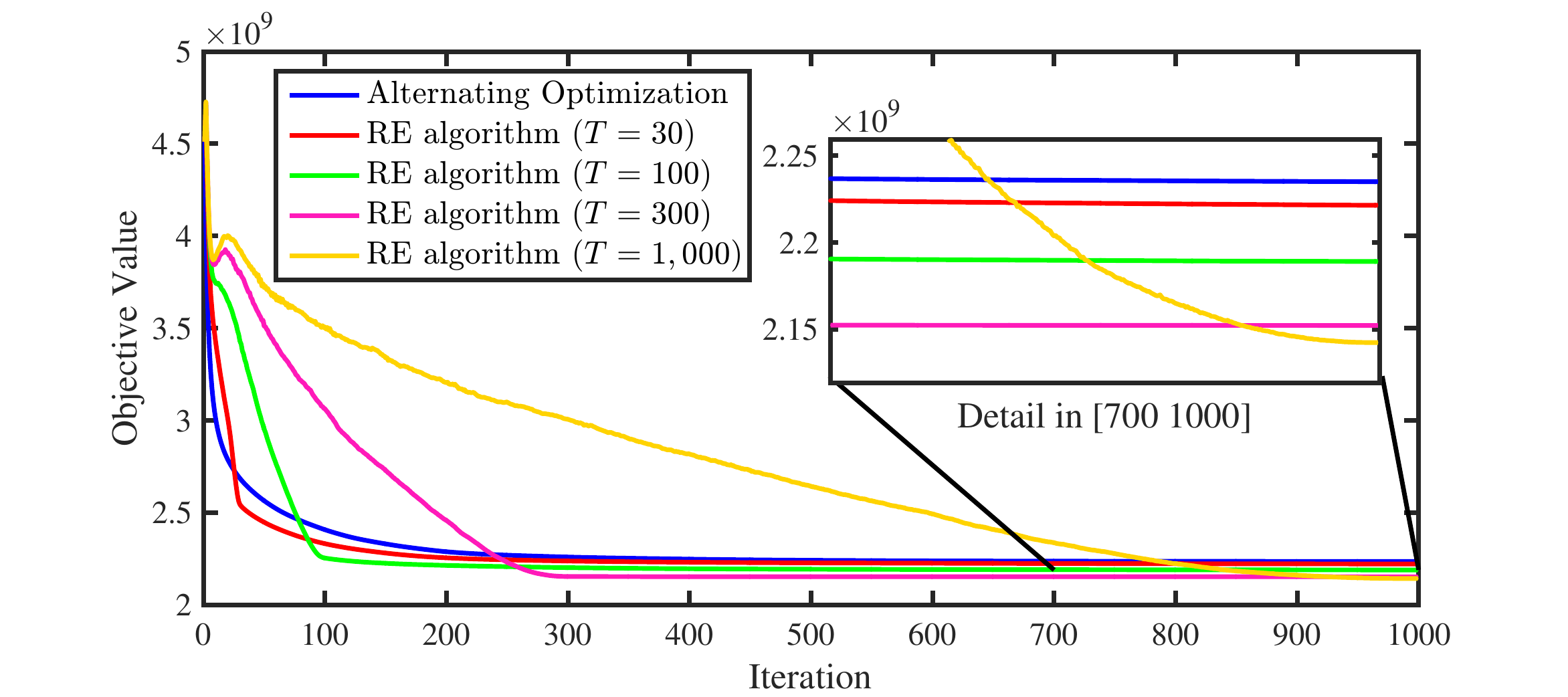}
				\vspace{-3mm}
				\caption{Objective value of \Eref{eq:OPQ} versus iteration number of OPQ optimization. We report average results over five different initializations.}
				\label{fig:OPQObjVsIter}
				\vspace{-3mm}
			\end{figure}

	We show the results in \Tref{tbl:DeblurResults1}. Our method significantly outperforms Pan's method~\cite{pan2014deblurring} and is successful for a significantly blurred image, as in \Fref{fig:DeblurFig}. We found that \Aref{alg:REalgorithm2} is superior to \Aref{alg:REalgorithm} in the cases of \{image \#3, kernel \#4\} and \{image \#5, kernel \#4\}. Note that these results are obtained by minimizing the same objective function, however using different optimization methods. Therefore our method likely improves upon other methods which use different objective functions~\cite{krishnan2011blind,xu2013unnatural}.

	\section{Conclusion}
	We proposed the RE algorithm, which is a novel global optimization algorithm for nonconvex LS problems. This method is based on a novel measurement of global convergence called RE convergence. We presented theoretical analysis of RE convergence and empirical results showing excellent performance of the RE algorithm for various optimization problems.
	
	There remain many open questions in both theoretical and empirical aspects. We can guarantee that the solution with the largest RE constant is the global optimum in limited cases. To which problems this applies remains unknown. We plan to investigate the applicability of the RE algorithm to other nonconvex optimization problems, including non-LS problems.
	
	\section*{Acknowledgement}
	This research is partially supportted by CREST (JPMJCR1686)
	and KAKENHI（15K12025）

{\small
\bibliographystyle{ieee}
\bibliography{egbib}
}

\end{document}